\documentclass[10pt]{article}
\usepackage{graphicx} % Required for inserting images
\usepackage{amsmath}
\usepackage{amsthm}
\usepackage{amssymb,color}
\usepackage[sort,compress]{natbib}

\usepackage{mathtools,enumitem}

\usepackage{bm}
\usepackage{comment}
\usepackage{geometry}
\newtheorem{theorem}{Theorem}

\newtheorem{lemma}[theorem]{Lemma}
\newtheorem{definition}[theorem]{Definition}

\newtheorem{remark}{Remark}

\definecolor{yxc}{RGB}{255,0,0}

\allowdisplaybreaks

\ifdefined\usebigfont

\usepackage{times}
\usepackage[fontsize=13pt]{scrextend}
\makeatletter
\usepackage[unicode=true,
 bookmarks=false,
 breaklinks=false,pdfborder={0 0 1},colorlinks=false]
 {hyperref}
\hypersetup{colorlinks,citecolor=blue,filecolor=blue,linkcolor=blue,urlcolor=blue}

\else
\usepackage[unicode=true,
 bookmarks=false,
 breaklinks=false,pdfborder={0 0 1},colorlinks=false]
 {hyperref}
\hypersetup{colorlinks,citecolor=blue,filecolor=blue,linkcolor=blue,urlcolor=blue}
\fi

\title{Anytime Acceleration of Gradient Descent}
\author{%
	Zihan Zhang\thanks{Paul G. Allen School of Computer Science and Engineering, University of Washington; email: \texttt{zihanz46@uw.edu, ssdu@cs.washington.edu}.}\\
 U.~Washington 
   \and
Jason D.~Lee\thanks{Department of Electrical and Computer Engineering, Princeton University; email: \texttt{jasonlee@princeton.edu}.}\\
 Princeton 
 \and
 Simon S.~Du\footnotemark[1]\\
 U.~Washington
 \and
 Yuxin Chen\thanks{Department of Statistics and Data Science, University of Pennsylvania; email: \texttt{yuxinc@wharton.upenn.edu}.}\\
 UPenn
}
\date{November 2024; ~Revised: December 2024}

\begin{document}
\maketitle

\begin{abstract}
This work investigates stepsize-based acceleration of gradient descent with {\em anytime} convergence guarantees. 
For smooth (non-strongly) convex optimization, we propose a stepsize schedule that allows gradient descent to achieve convergence guarantees of 
%$O(T^{-1.03})$ 
$O\big(T^{-\frac{2\log_2\rho}{1+\log_2\rho}}\big) \approx O(T^{-1.119})$ 
for any stopping time  $T$, 
where $\rho=\sqrt{2}+1$ is the silver ratio and  
the stepsize schedule is predetermined without prior knowledge of the stopping time. This result provides an affirmative answer to a COLT open problem  \citep{kornowski2024open} regarding whether stepsize-based acceleration  can yield anytime convergence rates of $o(T^{-1})$. We further extend our theory to yield anytime convergence guarantees of $\exp(-\Omega(T/\kappa^{0.893}))$ for smooth and strongly convex optimization, with $\kappa$ being the condition number. 
\end{abstract}

\tableofcontents

\section{Introduction}

Consider the standard problem of smooth convex optimization: 
\begin{align}
\mathop{\text{minimize}}\limits_{\bm{x}\in \mathbb{R}^d}~~f(\bm{x}),
\label{eq:convex-smooth}
\end{align}
where $f: \mathbb{R}^d\rightarrow \mathbb{R}$ is  smooth and convex (but not necessarily strongly convex). We assume without loss of generality that $f(\cdot)$ is 1-smooth (i.e., $\nabla f(\cdot)$ is 1-Lipschitz). 
In addition, we denote by $\bm{x}^*$ a minimizer of \eqref{eq:convex-smooth},  and set $f^*=f(\bm{x}^*)$. Our focal point is the classical gradient descent (GD) algorithm:
\begin{equation}
\bm{x}_{t+1} = \bm{x}_{t}-\alpha_t \nabla f(\bm{x}_t),
\qquad t\in \mathbb{N},
\label{eq:GD}
\end{equation}
where $\alpha_t>0$ stands for the stepsize at iteration $t$, and $\bm{x}_0$ denotes the initialization.

Textbook gradient descent theory typically recommends a constant stepsize schedule $\alpha_t \equiv \alpha \in (0,2)$, 
which ensures monotonicity of the objective value and guarantees that $f(\bm{x}_T)-f^*\leq O(1/T)$ for any stopping time $T$  \citep{nesterov2018lectures}. 
Somewhat surprisingly, a recent strand of work \citep{teboulle2023elementary,altschuler2023acceleration2,altschuler2023acceleration,altschuler2018greed,grimmer2024provably,grimmer2023accelerated,rotaru2024exact,grimmer2024accelerated} uncovered that adopting a time-varying stepsize schedule with occasional long steps can provably accelerate GD, achieving a convergence rate as fast as \citep{altschuler2023acceleration2,grimmer2024accelerated} 
%\textcolor{red}{Zihan: The convergence rate of Grimmer et al., 2023 is $O(T^{-1.0564})$ for fixed $T$. I am not sure whether they improve the result in a letter version.}
\begin{equation}
    f(\bm{x}_T)-f^*\leq 
    O\big(T^{-\log_2\rho}\big)
    \quad 
    \text{if }T=2^k-1\text{ for some }k\in \mathbb{N}_+,
\label{eq:convergence-theory-silver}
\end{equation}
where 
$$\rho \coloneqq 1+\sqrt{2}$$ is the silver ratio and $\log_2\rho \approx 1.2716$. 
As a concrete example, this stepsize-based acceleration \eqref{eq:convergence-theory-silver} is achievable via the so-called {\em silver stepsize schedule} \citep{altschuler2023acceleration2}, which is constructed recursively and incorporates some large stepsizes far exceeding $2$.

While occasional huge steps suffice in speeding up GD, 
the convergence guarantees \eqref{eq:convergence-theory-silver} proven by \citet{altschuler2023acceleration2,grimmer2023accelerated} only hold 
for exponentially increasing stopping times (i.e., $T=2^k - 1$ for  $k\in \mathbb{N}_+$). 
Given the non-monotonicity of $f(\bm{x}_t)$ in $t$ due to the adoption of long steps, 
the intermediate points (i.e., those not corresponding to $t=2^k-1$) might incur significant sub-optimality gaps. 
In fact, it has been shown by \citet[Corollary~4]{kornowski2024open}  that the silver stepsize schedule cannot even guarantee $f(\bm{x}_t)-f^{*}\rightarrow 0$ at intermediate iterations. 

To remedy this issue,  \citet{grimmer2024composing,zhang2024accelerated} proposed  improved stepsize construction strategies that achieve $f(\bm{x}_T)-f^*\leq O(T^{-\log_2\rho})$ for a prescribed stopping time $T$. 
One limitation  of this approach is that it requires the stopping time $T$ to be known in advance, as the stepsize schedule is designed based on the specific value of $T$. 
In practice, however, there is no shortage of applications where the stopping time is not predetermined and might vary during the execution of the algorithm. 
This gives rise to the following natural question,  posed by \citet{kornowski2024open}  at COLT 2024 as an open problem: 
{
\setlist{rightmargin=\leftmargin}
\begin{itemize}
    \item[] {\em  \textbf{Question:} Is there a stepsize schedule $\{\alpha_t\}_{t=1}^{\infty}$ that allows GD to achieve $f(\bm{x}_T)-f^*\leq 
    o(1/T)$ for any stopping time $T\in \mathbb{N}$, where $\{\alpha_t\}_{t=1}^{\infty}$ is constructed without prior knowledge of $T$?  }
\end{itemize}
}
\noindent 
In other words, this open problem asks whether it is feasible to achieve {\em anytime} convergence guarantees for GD that improve upon the textbook rate $O(1/T)$.

\paragraph{Overview of our results.}
In this work, we answer the above-mentioned open problem affirmatively. Our main finding is summarized below.
\begin{theorem}\label{thm:main} 
There exists a stepsize schedule  $ \{\alpha_t\}_{t=1}^{\infty}$, 
generated without knowing the stopping time, 
such that the gradient descent iterates \eqref{eq:GD} obey\footnote{Throughout this paper, we use $\|\cdot\|$ to denote the $\ell_2$ norm.}
\begin{align}
f(\bm{x}_T) -f^* \leq  O\bigg(\frac{\|\bm{x}_1-\bm{x}^*\|^2}{T^{\vartheta}} \bigg) \qquad
\text{with }\vartheta =  \frac{2\log_2\rho   }{1+\log_2\rho} \approx 1.119
\end{align}
for an arbitrary stopping time $T\geq 1$.
\end{theorem} 
To the best of our knowledge, 
our result provides the first stepsize schedule that provably accelerates gradient descent in an anytime fashion.
%
%\yxc{add maybe 2 sentences describing the stepsize schedule.}
%
The proposed stepsize schedule is inspired by, and constructed recursively based upon, 
the stepsize concatenation strategy recently proposed by \citet{zhang2024accelerated} (see also \citet{grimmer2024composing}). 
A key ingredient underlying our algorithm design is to ensure that the sizes of the gradients in intermediate iterations are well-controlled, 
so that the intermediate steps do not overshoot.

\paragraph{Other related work.}
%\input{related_work}

\iffalse
Besides the most related works presented above, we introduce some other papers working on designing stepsize schedule to improve the performance of gradient descent. 
\cite{drori2014performance} proposed PEP (performance estimation problem) to find a tighter bound for the worst-case performance of constant stepsize schedule.
\cite{taylor2017smooth} later developed necessary and sufficient conditions for interpolations of smooth and convex functions.
\cite{das2024branch} worked on finding the best possible worst-case convergence rate by solving the PEP with a branch and bound method.

To improve the constant in the $O(1/T)$ convergence rate,
\cite{teboulle2023elementary} proposed a dynamic bounded stepsize schedule, and \cite{grimmer2024provably} considered the periodic stepsize schedule. Both methods achieves a non-trivial constant improvement. On the other side,
\cite{rotaru2024exact}  studied the worst-case convergence rate for constant stepsize schedules for smooth non-convex functions, and proved that the norm of gradient 
 
There also exists a series of paper \cite{altschuler2018greed,daccache2019performance,eloi2022worst}  working on computing the exact worst-case performance of gradient descent for fixed small $t$.
Most of the previous work focus on improving the worst-case convergence rate for  a fixed step $t$, instead of pursuing an any-time convergence rate.
\fi 

%Jason thesis %\cite{altschuler2018greed}

%strongly convex: \cite{altschuler2023acceleration}

In addition to the most relevant work described above, we mention in passing several other papers on gradient descent acceleration.  
\citet{drori2014performance} proposed the performance estimation problem (PEP) to identify tighter bounds on  the worst-case GD performance under constant stepsize schedules.
\citet{taylor2017smooth}  put forward closed-form necessary and sufficient conditions for smooth  (strongly) convex interpolation, offering a finite representation of these functions. 
\citet{das2024branch} attempted to find the best possible worst-case convergence rate by solving the PEP via a branch-and-bound method. 
To improve the pre-constant in the $O(1/T)$ convergence rate,
\citet{teboulle2023elementary} proposed a dynamic bounded stepsize schedule, and \citet{grimmer2024provably} considered the periodic stepsize schedule. Both methods achieve highly non-trivial constant improvements. Additionally,
\citet{rotaru2024exact}  studied the worst-case convergence rate for constant stepsize schedules for smooth non-convex functions, and established better convergence rates for weakly convex problems.
There have also been a series of papers \citep{altschuler2018greed,daccache2019performance,eloi2022worst}  that computed the exact worst-case performance of GD for some fixed small iteration $t$.
Noteworthily, most of the previous work focused on improving the worst-case convergence guarantees for  a given stopping time $T$, instead of pursuing  acceleration in an any-time fashion.

\paragraph{Paper organization.} Section~\ref{sec:pre} introduces some basics about GD, as well as  useful results from \cite{zhang2024accelerated} concerning the so-called ``primitive stepsize schedule.''  Construction of the proposed stepsize schedule and the proof of Theorem~\ref{thm:main}  provided in Section~\ref{sec:analysis}. In Section~\ref{sec:strong_cvx}, we further extend our result to accommodate smooth and strongly convex optimization. 

%optimizing $\mu$-strong convex functions, which beats the $\exp(-1/\kappa)$ any-time convergence  $(\kappa  = 1/\mu)$.

\paragraph{Notation.}
We also introduce a couple of notation to be used throughout. 
%
%We use $\bm{g}(\bm{x})$ to denote $\nabla f(\bm{x})$. 
%Let the trajectory of the optimization process be $\{x_i\}_{i\geq 1}$. That is, give a stepsize schedule $\{\alpha_i\}_{i\geq 1}$ , we define $\bm{x}_{i+1}=\bm{x}_i -\alpha_i g(\bm{x}_i)$  for $i\geq 1$.
%Let $\bm{x}^*$ be the optimal solution. 
Denote by $\bm{1}$ the all-one vector with compatible dimension. 
Set 
\begin{equation}
%f^* = f(\bm{x}^*),\qquad 
f_i = f(\bm{x}_i)
\qquad \text{and} \qquad 
\bm{g}_i =  \nabla f(\bm{x}_i)
\label{eq:defn-fi-gi}
\end{equation}
for each iteration $i$. 
For a given stepsize schedule $\{\alpha_t\}_{t\geq 1}$, we set
\begin{align}
A_n \coloneqq \sum_{i=1}^{n-1}\alpha_{i}
\qquad \text{and}\qquad C_n \coloneqq  \frac{A_n (A_n+1)}{2}
\label{eq:defn-An-Cn}
\end{align}
for any integer $n\geq 2$, 
where in the notation of $A_n$ and $C_n$, we suppress the dependence on $\{\alpha_t\}_{t\geq 1}$ as long as it is clear from the context.   
Additionally, for an infinite sequence $\bm{r}=[r_j]_{j= 1}^{\infty}$, 
we define 
\begin{align}
A_n(\bm{r})= \sum_{i=1}^{n-1}r_i
\qquad \text{and}\qquad 
C_n(\bm{r})= \frac{A_n(\bm{r})(A_n(\bm{r})+1)}{2}. 
%$\overline{A}(\bm{s})=\overline{A}$
\label{eq:defn-An-Cn-inf}
\end{align}
In addition, we often use $\bm{\alpha}_{\ell:k}$ to indicate the stepsize subsequence $[\alpha_{\ell},\dots,\alpha_k]^{\top}$, and let $\alpha_i(\bm{s})$ denote the $i$-th stepsize in a stepsize sequence $\bm{s}$. 

%\yxc{TODO}
%We also set $\overline{A}(\bm{s}) = \sum_{i=1}^{k-1}\alpha_{1:k-1}(\bm{s})$When  the sequence $\bm{s}=\bm{\alpha}_{1:k-1}$ is clear in the context, we write $A_n(\bm{s})=A_n$, $\overline{A}(\bm{s})=\overline{A}$ and $C_n(\bm{s})=C_n$ for simplicity.

\section{Preliminaries}\label{sec:pre}

\paragraph{Basic inequalities for smooth convex functions.} 

Let us gather a set of elementary inequalities for a 1-smooth convex function $f(\cdot)$: 
\begin{subequations}
\label{eq:basic-inequalities-f}
\begin{align}
 f_i -f^* - \langle \bm{g}_i, \bm{x}_i-\bm{x}^* \rangle + \frac{1}{2} \|\bm{g}_i\|^2 &\leq 0,
 \label{eq:fi-fstar-ineq}
\\  
f^* -f_i +\frac{1}{2}\|\bm{g}_i\|^2 &\leq 0,
\\  f_i -f_j - \langle \bm{g}_i, \bm{x}_i-\bm{x}_j \rangle + \frac{1}{2}\|\bm{g}_i-\bm{g}_j\|^2  &\leq 0 ,
\\  f_j -f_i- \langle \bm{g}_j, \bm{x}_j-\bm{x}_i \rangle + \frac{1}{2}\|\bm{g}_i-\bm{g}_j\|^2 &\leq 0,
\end{align}
and for any $\bm{x}$ and $\alpha> 0$,
\begin{align}
&f\big(\bm{x}-\alpha\nabla f(\bm{x})\big)-f(\bm{x}) \nonumber\\
&\qquad\leq\alpha\big\langle\nabla f\big(\bm{x}-\alpha\nabla f(\bm{x})\big),\nabla f(\bm{x})\big\rangle-\frac{1}{2}\big\|\nabla f(\bm{x})-\nabla f\big(\bm{x}-\alpha\nabla f(\bm{x})\big)\big\|^{2}.
\label{eq:1smooth0-diff}
\end{align}
See, e.g., \citet{beck2017first} or \citet[Section~2.1]{zhang2024accelerated} for proofs of these well-known facts. 
In addition, given that $\alpha\langle\bm{a},\bm{b}\rangle=\alpha\|\bm{b}\|^{2}+\alpha\langle\bm{a}-\bm{b},\bm{b}\rangle\leq\alpha\|\bm{b}\|^{2}+\frac{\alpha^{2}}{2}\|\bm{b}\|^{2}+\frac{1}{2}\|\bm{a}-\bm{b}\|^{2}$ (a consequence of the Cauchy-Schwarz inequality), we can further upper bound 
\eqref{eq:1smooth0-diff} by
\begin{align}
f\big(\bm{x}-\alpha\nabla f(\bm{x})\big)-f(\bm{x}) & \leq\frac{\alpha^{2}+2\alpha}{2}\|\nabla f(\bm{x})\|^{2}
\qquad \forall \alpha >0
\text{ and }
\bm{x} .
\label{eq:local_error}
\end{align}
%
%for any $\bm{x}$ and $\alpha\geq 0$. 
\end{subequations}

%\begin{align}
% & f\big(\bm{x}-\lambda \nabla f(\bm{x})\big)- f(\bm{x}) \nonumber
% \\ & \leq \lambda \nabla f\left(\bm{x}-\lambda \nabla f(\bm{x}\right)\cdot  \nabla f(\bm{x}) - \frac{1}{2}\left( \nabla f(\bm{x})- \nabla f(\bm{x}-\lambda \nabla f(\bm{x})) \right)^2\nonumber
% \\ &\leq  \frac{\lambda^2 + 2\lambda}{2}\|\nabla f(\bm{x})\|^2 \label{eq:local_error}
%\end{align}
%for all $\bm{x}$ and $\lambda\geq 0$.

\paragraph{Primitive stepsize schedule and concatenation.} 
Next, we formalize the notion of  ``primitive stepsize schedule'' as introduced in \citet[Definition 3]{zhang2024accelerated}. 
\begin{definition}
[Primitive stepsize schedule]\label{defn:primitive}
A stepsize schedule $\bm{\alpha}_{1:k-1}=[\alpha_1,\dots,\alpha_{k-1}]\in \mathbb{R}_+^{k-1}$ is said to be primitive if
\begin{subequations}
\begin{align}&
A_{k}(f_k-f^*) + C_k \|\bm{g}_k\|^2 +\frac{1}{2}\|\bm{x}_k-\bm{x}^*\|^2
\notag\\ &\quad\leq \frac{1}{2}\|\bm{x}_1-\bm{x}^*\|^2 +\sum_{i=1}^{k-1} \alpha_i \left(f_i - f^* -\langle \bm{g}_i, \bm{x}_i-\bm{x}^*\rangle + \frac{1}{2}\|\bm{g}_i\|^2 \right)\label{eq:defn-primitive-1}
\\  &\quad \leq \frac{1}{2}\|\bm{x}_1-\bm{x}^*\|^2,
\label{eq:defn-primitive}
\end{align}
\end{subequations}
where we recall the definition of $A_k$ and $C_k$ in \eqref{eq:defn-An-Cn}. 
\end{definition}
When $k=1$, \eqref{eq:defn-primitive-1} holds trivially, which means that the null sequence is a primitive stepsize schedule .
As it turns out, 
two primitive stepsize schedules can be concatenated to form a longer primitive sequence, which forms the basis for the convergence guarantees in \citet{zhang2024accelerated}. 
The following lemma, derived by  \citet{zhang2024accelerated}, makes precise this key property; 
for completeness, we provide a proof  in Appendix~\ref{sec:proof-lemma:key1}. 
%
%We state a key lemma from the recent paper \citet{zhang2024accelerated} as follows.  In high-level idea, this lemma shows that two sequences with good \emph{final-step} convergence rate could be concatenated into a longer sequence with good \emph{final-step} convergence rate. 
%

%
\begin{lemma}\label{lemma:key1}(\citet[Theorem 3.1]{zhang2024accelerated})
Consider a stepsize schedule 
$\{\alpha_t\}_{t\geq 1}$. 
Suppose that 
 both $\bm{\alpha}_{1:\ell-1}=[\alpha_1,\dots,\alpha_{\ell-1}]^{\top}$ and $\bm{\alpha}_{\ell+1:k-1}= [\alpha_{\ell+1},\dots,\alpha_{k-1}]^{\top}$ are primitive. 
Define the following function
\begin{equation}
\varphi(x,y) \coloneqq \frac{-(x+y)+\sqrt{  (x+y+2)^2 +4(x+1)(y+1)
 }}{2}.
\end{equation}
Then, $\bm{\alpha}_{1:k-1}=[\alpha_1,\dots,\alpha_{k-1}]$ is also primitive if
$$
\alpha_{\ell}=\varphi\big(
\bm{1}^{\top}\bm{\alpha}_{1:\ell-1}, \bm{1}^{\top}\bm{\alpha}_{\ell+1:k-1}\big).
$$
%
%Set $x = A_{\ell} = \sum_{i=1}^{\ell-1}\alpha_i$ and $y = A_{k}-A_{\ell+1} = \sum_{i=\ell+1}^{k-1}\alpha_i$. Set $\alpha = \varphi(x,y) = \frac{-(x+y)+\sqrt{  (x+y+2)^2 +4(x+1)(y+1) }}{2}$. Then the sequence $\alpha_{1:k-1}$ is also a primitive sequence by setting $\alpha_{\ell}=\alpha$.
\end{lemma}
%
%\noindent 
%This lemma was established by \citet{zhang2024accelerated}. 
%
\noindent 
With Lemma~\ref{lemma:key1} in mind, we find it convenient to introduce the concatenation function as follows: for any two nonnegative vectors $\bm{s}$ and $\bm{r}$, define 
\begin{align}
\mathsf{concat}(\bm{s},\bm{r})
\coloneqq 
\big[\bm{s}^{\top},\varphi(\bm{1}^{\top}\bm{s},\bm{1}^{\top}\bm{r}),\bm{r}^{\top}\big]^{\top}.
\label{eq:defn-concat}
\end{align}

As an immediate consequence, 
if we have available a collection of basic primitive sequences --- denoted by $\{\bm{s}_i\}_{i\geq 1}$, then we can concatenate them as follows: 
\begin{subequations}
\label{eq:defn-s-hat}
\begin{align}
\widehat{\bm{s}}_{0} & =[\,],\\
\widehat{\bm{s}}_{i} & \leftarrow \mathsf{concat}(\widehat{\bm{s}}_{i-1},\bm{s}_{i}),\qquad i=1,2,\dots\\
\widehat{\bm{s}} & \leftarrow\lim_{i\rightarrow\infty}\widehat{\bm{s}}_{i}.
\end{align}
\end{subequations}
The resulting $\widehat{\bm{s}}$ is well-defined and primitive, as asserted by the following lemma. 
%We will first design some basic primitive sequences $\{\bm{s}_i\}_{i\geq 1}$ and then concatenate these sequences in the following way.
\begin{lemma}\label{lemma:well-define} 
Suppose that each $\bm{s}_i$ ($i\geq 1$) is primitive. 
Then each $\widehat{\bm{s}}_i$ ($i\geq 1$) is primitive,  
and the infinite sequence $\widehat{\bm{s}}$  is well-defined and primitive. 
%, and each . 
%
%Fix  $\{\bm{s}_i\}_{i\geq 1}$. Let $\widehat{\bm{s}}_0$ be a null sequence and  $\widehat{\bm{s}}_i = \mathsf{concat}(\widehat{\bm{s}}_{i-1},\bm{s}_i)$ for all $i\geq 1$. Then the infinite sequence $\widehat{\bm{s}} \coloneqq \lim_{i\to \infty}\widehat{\bm{s}}_i$ is well-defined. Moreover, if $s_i$ is a primitive sequence for all $i\geq 1$, $\widehat{\bm{s}}_i$ is also a primitive sequence for all $i\geq 1$.
\end{lemma}
\begin{proof}
For each $i\geq 1$, $\widehat{\bm{s}}_{i-1}$ is always a prefix of $\mathsf{concat}(\widehat{\bm{s}}_{i-1}, \bm{s}_i )=\widehat{\bm{s}}_{i}$. As a result, for any $n\geq 1$, the $n$-th element of $ \lim_{i\to \infty}\widehat{\bm{s}}_i$ exists, and hence $\widehat{\bm{s}}$ is well-defined.

Additionally, note that the null $\widehat{\bm{s}}_0$ is  primitive. Assuming that $\widehat{\bm{s}}_{i-1}$ is primitive for some $i\geq 1$, 
we see from Lemma~\ref{lemma:key1} that $\widehat{\bm{s}}_i = \mathsf{concat}(\widehat{\bm{s}}_{i-1},\bm{s}_i)$ is also primitive. Therefore, an induction argument shows that $\widehat{\bm{s}}_i$ is  primitive for every $i\geq 1$, and so is $\widehat{\bm{s}}$.
\end{proof}

\paragraph{Silver stepsize schedule.}
We now introduce the silver stepsize schedule proposed by \citep{altschuler2023acceleration2}. %
\begin{definition}[Silver stepsize schedule]\label{def:silver}
Let $\overline{\bm{s}}_0=[\,]$ be the null sequence, and set $\overline{\bm{s}}_{i} = \mathsf{concat}(\overline{\bm{s}}_{i-1},\overline{\bm{s}}_{i-1})$ for each $i\geq 1$. Then 
$\overline{\bm{s}}_i$ is said to be the $i$-th order silver stepsize schedule, 
with the (limiting) silver stepsize schedule given by $\overline{\bm{s}}\coloneqq \lim_{i\to\infty}\overline{\bm{s}}_i$. 
\end{definition}
Given that $\overline{\bm{s}}_i$ is always a prefix of $\overline{\bm{s}}_{i+1} = \mathsf{concat}(\overline{\bm{s}}_{i},\overline{\bm{s}}_{i}) $ for each $i\geq 0$, the limiting  $\lim_{i\to\infty}\overline{\bm{s}}_i$ exists and hence $\overline{\bm{s}}$ is well-defined. 
Moreover, we single out the following properties about the silver stepsize schedule.
\begin{lemma}\label{lemma:basic_property}
For each $i\geq 1$, $\overline{\bm{s}}_{i}$ is a primitive sequence with length $2^{i}-1$. Moreover, it holds that
%by writing $\overline{\bm{s}}_i = \alpha_{1:2^i-1}(\overline{\bm{s}}_i)$, we have that 
%
\begin{equation}
%\sum_{j=1}^{2^i-1}\alpha_j(\overline{\bm{s}}_i) 
\bm{1}^{\top} \overline{\bm{s}}_i = \rho^i-1,
\qquad i=0,1,\dots
\label{eq:si-length-seq}
\end{equation}
where we recall that $\rho = \sqrt{2}+1$.
\end{lemma}
\begin{proof}
First of all, 
Lemma~\ref{lemma:key1} tells us that $\overline{\bm{s}}_{k+1}=\mathsf{concat}(\overline{\bm{s}}_k,\overline{\bm{s}}_k)$ is primitive as long as 
$\overline{\bm{s}}_k$ is primitive. 
Given that $\overline{\bm{s}}_0=[\,]$ is also primitive, we can prove by induction that $\overline{\bm{s}}_i$ is primitive for every $i\geq 1$.

Next, we prove \eqref{eq:si-length-seq} by induction. To begin with, 
the claim \eqref{eq:si-length-seq} is trivial for $i=0$. 
Now assuming that \eqref{eq:si-length-seq}  holds for $k$,  
we have
\begin{align}
\bm{1}^{\top}\overline{\bm{s}}_{k+1} & =2(\bm{1}^{\top}\overline{\bm{s}}_{i})+\varphi(\bm{1}^{\top}\overline{\bm{s}}_{k},\bm{1}^{\top}\overline{\bm{s}}_{k})\notag\\
 & =2(\rho^{k}-1)+\big\{(\sqrt{2}-1)(\rho^{k}-1)+\sqrt{2}\big\}\nonumber\\
 & =\rho(\rho^{k}-1)+\sqrt{2}=\rho^{k+1}-1,
    % D_{k+1}=2D_{k}  +\varphi(D_k,D_k) &=  2D_k + (\sqrt{2}-1)D_k + \sqrt{2} = (\sqrt{2}+1)D_k + \sqrt{2} \nonumber
     %\\ & \qquad \qquad \qquad\qquad  = \rho (\rho^k -1)+\sqrt{2} = \rho^{k+1}-1\nonumber
\end{align}
which justifies \eqref{eq:si-length-seq} for $i+1$. This establishes \eqref{eq:si-length-seq} by induction. 
%
%   Let $D_i = \sum_{j=1}^{2^i - 1}\alpha_j (\overline{\bm{s}}_i)$ for $i\geq 1$. Because $\overline{\bm{s}}_k$ is a primitive sequence, with Lemma~\ref{lemma:key1}, we learn that $\overline{\bm{s}}_{k+1}=\mathsf{concat}(\overline{\bm{s}}_k,\overline{\bm{s}}_k)$ is also a primitive sequence. For the second claim, we compute $D_{k+1}$ as 
%
 %   as claimed.
\end{proof}

%\begin{definition}\label{def:con} Given two sequences $\bm{s}_1$ and $\bm{s}_2$, the concatenate function $\mathtt{C}(\bm{s}_1,\bm{s}_2)$ is a sequence $\left[\bm{s}_1, \varphi\left(\overline{A}(\bm{s}_1), \overline{A}(\bm{s}_2)\right) ,\bm{s}_2\right]$
% \end{definition}

%\input{stepsize_construction.tex}

\section{Analysis}
\label{sec:analysis}

\subsection{Construction of our stepsize schedule}\label{sec:const} 

Armed with the silver stepsize schedules $\{\overline{\bm{s}}_j\}_{j\geq 0}$  introduced in Definition~\ref{def:silver} --- which serve   
as basic primitive sequences --- 
we can readily present the proposed stepsize schedule.

Choose some positive quantity $c\geq 1$. While we shall keep $c$ as a general quantity throughout most of the proof, it will be taken to be $c=\log_2\rho$ at the last step of our proof of the main theorems.  
%let $k_0=M_0=0$, 
Also, set
\begin{equation}
k_0=M_0=0,\qquad 
k_i = \left\lfloor2\cdot 2^{c i} \right\rfloor,
\qquad \text{and}\qquad 
M_i = \sum_{j=1}^i k_i,
\qquad i=1,2,\dots
\end{equation}
%
%for $i\geq 1$, where $c=7$. 
%In particular, we define $k_0 = M_0 =0$. 
%
With these parameters in place, 
our construction proceeds as follows:
\begin{itemize}
    \item For each $j\geq 1$, set $\bm{s}_i=\overline{\bm{s}}_j$ for every $i$ obeying $M_{j-1}<i \leq M_{j} $, 
    where $\overline{\bm{s}}_j$ denotes the $j$-th order silver stepsize schedule in Definition~\ref{def:silver}. 
    In other words,  we repeat $\overline{\bm{s}}_j$ for $k_j$ times for each $j\geq 1$, with $k_j$ exponentially increasing in $j$.

    \item Generate the infinite stepsize sequence $\widehat{\bm{s}}$ through the concatenation procedure in \eqref{eq:defn-s-hat}. 
    
\end{itemize}
Throughout the rest of the paper, we denote by $t_i$ the length of 
the $i$-th order subsequence $\widehat{\bm{s}}_i$, as constructed in \eqref{eq:defn-s-hat}. 

%Let $t_i$ be the end point of the $i$-th subsequence in $\widehat{\bm{s}}$. That is, $t_i$ is the length of $\widehat{\bm{s}}_i$.

%We define $\{\bm{s}_i\}_{i\geq 1}$ by setting $\bm{s}_i  = \overline{\bm{s}}_j$ for $M_{j-1}<i \leq M_{j} $ and all $j \geq 1$. In words, we repeat $\overline{\bm{s}}_i$ for $k_i$ times for each $i\geq 1$.

%Write the final sequence $\widehat{\bm{s}}$ as $\bm{\alpha}_{1:\infty}(\widehat{\bm{s}})$. 
%Let $t_i$ be the end point of the $i$-th subsequence in $\widehat{\bm{s}}$. That is, $t_i$ is the length of $\widehat{\bm{s}}_i$.

We immediately single out an important property of the constructed stepsize schedule $\widehat{\bm{s}}$. 
The proof is postponed to Appendix~\ref{sec:proof-lem-sequence-computation}.

\begin{lemma}\label{lemma:sequence_computation}
%Let $\widehat{\bm{s}}=[\alpha_i]_{i= 1}^{\infty}$. 
For any $t\geq 1$, it holds that 
\begin{equation}
\label{eq:At-lower-bound}
A_{t+1}(\widehat{\bm{s}})\geq \frac{1}{36}t^{\frac{c+\log_2{\rho}}{c+1}} ,
\end{equation}
where $A_{t}(\widehat{\bm{s}})$ is defined in \eqref{eq:defn-An-Cn-inf}. 
Moreover, letting $o_t$ denote the integer obeying $ \sum_{j=1}^{o_t-1} k_j 2^j <t\leq \sum_{j=1}^{o_t} k_j 2^j $, one has 
\begin{equation}
\label{eq:ot-upper-bound}
2^{o_t}\leq 2 t^{\frac{1}{c+1}}.
\end{equation}
\end{lemma}

%\end{proof}

\subsection{A glimpse of high-level ideas}

Let us take a moment to briefly point out two key aspects underlying our design and analysis of the stepsize schedule.

\paragraph{Stepsize concatenation via suitable join steps.}
As proven recently by \citet{zhang2024accelerated,grimmer2024composing}, certain desirable stepsize schedules with different lengths can be concatenated --- with a properly chosen join stepsize --- into a longer stepsize schedule while ensuring fast convergence at the last step, which motivates our design.  
To be more concrete, a desirable stepsize schedule of this kind is the primitive stepsize schedule, and it has been shown that a primitive stepsize schedule with length $t$ enjoys the convergence rate of $O\big(\frac{1}{\sum_{i:i<t} \alpha_i}\big)$ at the last step \citep{zhang2024accelerated}. 
As a result, if we recursively prolong the stepsize schedule by concatenating the current one with another primitive stepsize schedule, then 
the $O\big(\frac{1}{\sum_{i<t}\alpha_i}\big)$ convergence rate continues to hold at the last step. 
Notably, every concatenation operation requires inserting a join stepsize in the middle, which we illustrate in  Figure~\ref{fig1}. 
%
%holds for the final step $t$ after each concatenation, which we name as \textbf{join steps} (see Figure 1). 
%
As it turns out, there is a trade-off between the aggregate stepsize $\sum_{i: i<t}\alpha_i$ and the number of join steps, making it crucial to choose a proper number of join steps. Fortunately, there exists some simple stepsize schedule with $\Omega(t^{1-\epsilon_1})$ join steps and an aggregate stepsize $\Omega(t^{1+\epsilon_2})$ for some proper constants $\epsilon_1,\epsilon_2>0$, 
which enables a convergence rate  of $o(t^{-1})$ at each join step.

 \begin{figure}
  \caption{ \textbf{Left:} the first 128 steps of the silver stepsize schedule; \textbf{Right:} the first 128 steps of our stepsize schedule (with parameter $c$ adjusted for better illustration). 
  The \textbf{red bars} indicate the positions of the join steps.  
   The number of join steps in the first $t$ steps of the silver stepsize schedule is $\left\lfloor\log_2 t\right\rfloor$, whereas in our schedule, this number is roughly $\Omega(t^{\frac{\log_2\rho}{\log_2\rho+1}})$. \label{fig1}
  }
     \minipage{0.52\textwidth}
  \includegraphics[width=\linewidth]{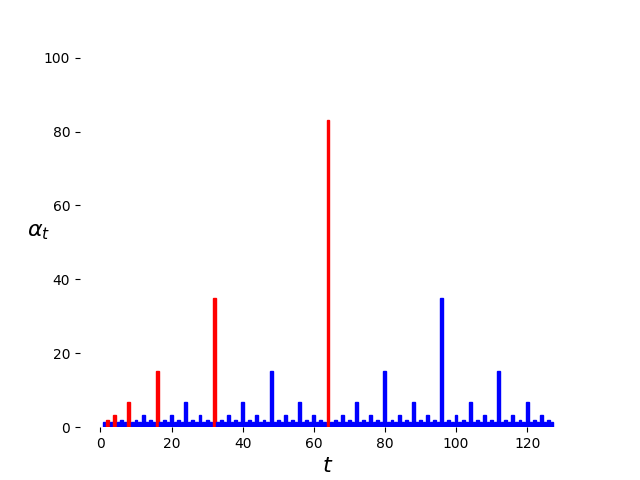}
\endminipage\hfill
 \minipage{0.52\textwidth}
  \includegraphics[width=\linewidth]{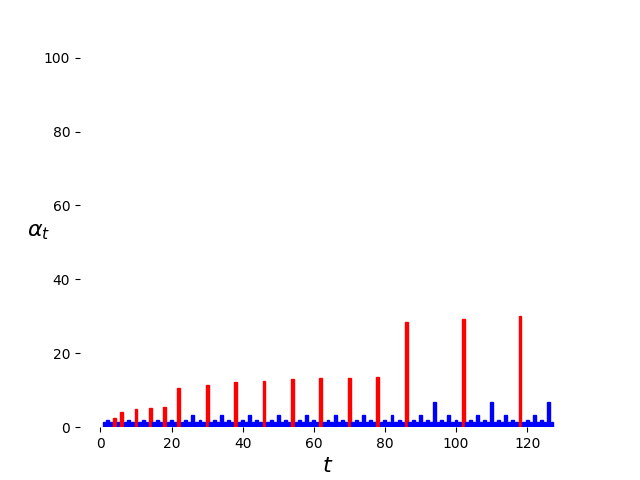}
\endminipage\hfill
 \end{figure}

\paragraph{Controlling the norm of gradients.}
While the above-mentioned concatenation strategy guarantees fast convergence at each join step, 
we still need to examine the convergence properties at intermediate steps (i.e., the ones between two adjacent join steps). 
Consider, for concreteness, iteration $\ell$, and denote by  $n_{\ell}$ the iteration number of the closest join step below $\ell$; see Figure~\ref{fig2} for an illustration. 
A common strategy to bound the difference $f_{\ell}-f_{n_{\ell}}$ of the associated objective values is to control the norm of the weighted gradients $\alpha_i^2 \|\bm{g}_i\|^2$ for every $i\in [n_{\ell}, \ell]$, which arises from the smoothness and convexity of $f$. A key part of our analysis thus boils down to bounding 
each $\alpha^2_i \|\bm{g}_i\|^2$ using the corresponding  weighted gradient at the join step $n_{\ell}$, 
for which the silver stepsize schedule enjoys some favorable property that enables effective control of the weighted gradient norm in this manner. 

%Owing to the results in \citet{grimmer2024accelerated,grimmer2024composing,zhang2024accelerated}, we can readily bound 
% $\alpha^2_i \|\bm{g}_i\|^2$ using the corresponding % weighted gradient at the join step $n_{\ell}$. 

% is well bounded for the join step $i=n_{\ell}$. 
% We then consider the following problem. Let $\alpha_{1:k-1}$ be a stepsize schedule with good convergence rate at the final step. Let $\alpha_0>0$ be a real number. Define the trajectory $x_{0:k}$ as be $\bm{x}_{i+1}=\bm{x}_i  -\alpha_i \bm{g}_i$ for $0\leq i \leq k-1$. Could we bound $\alpha_i^2 \|\bm{g}_i\|^2$ using a proper function of $\alpha_0^2 \|\bm{g}_0\|^2$ for any $1\leq i \leq k$?

% We show that in Lemma~\ref{lemma:s1}, by choosing $\alpha_{1:k-1}$ as a $i$-th order \emph{silver} stepsize schedule (see Definition~\ref{def:silver}), $\alpha_i^2 \|\bm{g}_i\|^2 = O( 16^i \alpha_0^2 \bm{g}_0^2 )$. This constructs an efficient bound for the gap $f_{\ell}-f_{n_{\ell}}$, and proves the $O(t^{-1.03})$ any-time convergence rate.

\begin{figure}\label{fig2}
\caption{An illustration of our analysis strategy to bound $f_{\ell}-f^*$ for an intermediate step $\ell$. Here, the \textbf{yellow point} indicates the initial step, whereas the \textbf{red points} indicate the join steps. Here, $n_{\ell}$ indicates the largest join step below $\ell$.}
\includegraphics[width=\linewidth]{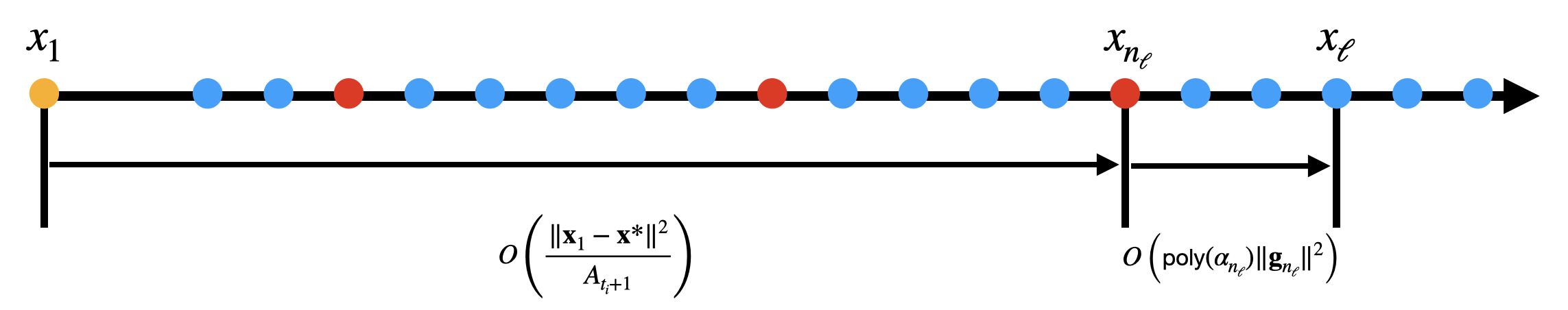}
\end{figure}

\subsection{Key lemmas}
Before proceeding to the proof of our main theorem, we single out a couple of key lemmas concerning the primitive stepsize schedule --- and in particular, the silver stepsize schedule --- that play an important role in our subsequent analysis.

The first lemma below singles out 
an important property of a primitive stepsize schedules, to be specified by \eqref{eq:rs1}. 

\begin{lemma}\label{lemma:key}
Suppose $\bm{s}=\bm{\alpha}_{1:k-1}$ is a primitive stepsize schedule. Then for any fixed $\bm{x}_0$ with gradient $\bm{g}_0$, it holds that
\begin{align}
&  A_{k} (f_{k}-f_0) + \frac{1}{2}\|\bm{x}_{k}-\bm{x}_0\|^2 + C_{k} \|\bm{g}_{k}\|^2\leq \frac{1}{2}\|\bm{x}_1-\bm{x}_0\|^2 + \sum_{i=1}^{k-1}\alpha_i \langle\bm{g}_i, \bm{g}_0\rangle - \frac{A_{k}}{2}\|\bm{g}_0\|^2;\label{eq:rs1}
\end{align}
\end{lemma}

\begin{proof}
 From  Definition~\ref{defn:primitive} of the primitive stepsize schedule, we obtain
\begin{align}
& A_k(f_k-f^*) + C_k \|\bm{g}_k\|^2  + \frac{1}{2}\|\bm{x}_k-\bm{x}^*\|^2 \nonumber
\\ & \leq \quad \frac{1}{2}\|\bm{x}_1-\bm{x}^*\|^2 +\sum_{i=1}^{k-1}\alpha_i \left( f_i -f^* -\left\langle \bm{g}_i, \bm{x}_i-\bm{x}^* \right\rangle +\frac{1}{2}\|\bm{g}_i\|^2 \right).\label{eq:qr1}
\end{align}
Also, the basic properties about smooth convex functions (cf.~\eqref{eq:basic-inequalities-f}) give
\begin{align}
\sum_{i=1}^{k-1} \alpha_i \left( f_i-f_0 -\left\langle \bm{g}_i, \bm{x}_i -\bm{x}_0 \right\rangle + \frac{1}{2}\|\bm{g}_i-\bm{g}_0\|^2  \right) \leq 0 ,\label{eq:qr2}
\end{align}
which further implies that
\begin{align}
 & \sum_{i=1}^{k-1}\alpha_i \left( f_i -f^* -\left\langle \bm{g}_i, \bm{x}_i-\bm{x}^* \right\rangle +\frac{1}{2}\|\bm{g}_i\|^2 \right) \nonumber
 \\ & \leq  \sum_{i=1}^{k-1}\alpha_i \left( f_i -f^* -\left\langle \bm{g}_i, \bm{x}_i-\bm{x}^* \right\rangle +\frac{1}{2}\|\bm{g}_i\|^2 \right)  - \sum_{i=1}^{k-1} \alpha_i \left( f_i-f_0 -\left\langle \bm{g}_i, \bm{x}_i -\bm{x}_0 \right\rangle + \frac{1}{2}\|\bm{g}_i-\bm{g}_0\|^2  \right)\nonumber
 \\ & = \sum_{i=1}^{k-1}\alpha_i \left(      f_0-f^* - \left\langle  \bm{g}_i ,\bm{x}_0-\bm{x}^* \right\rangle + \left\langle \bm{g}_i,\bm{g}_0 \right\rangle - \frac{1}{2}\|\bm{g}_0\|^2 \right).\label{eq:qr3}
\end{align}
Substituting \eqref{eq:qr3} into \eqref{eq:qr1}  and using the fact that  $\sum_{i=1}^{k-1}\alpha_i \bm{g}_i = \bm{x}_1-\bm{x}_k$, we can derive
\begin{align}
& A_k(f_k-f^*) + C_k \|\bm{g}_k\|^2  + \frac{1}{2}\|\bm{x}_k-\bm{x}^*\|^2 \nonumber
\\ & \leq  \frac{1}{2}\|\bm{x}_1-\bm{x}^*\|^2 +  \sum_{i=1}^{k-1}\alpha_i \left(      f_0-f^* - \left\langle  \bm{g}_i ,\bm{x}_0-\bm{x}^* \right\rangle + \left\langle \bm{g}_i,\bm{g}_0 \right\rangle - \frac{1}{2}\|\bm{g}_0\|^2 \right)\nonumber
\\ & = \frac{1}{2}\|\bm{x}_1-\bm{x}^*\|^2 +A_k(f_0-f^*) - \sum_{i=1}^{k-1}\alpha_i\left\langle \bm{g}_i, \bm{x}_0-\bm{x}^* \right\rangle + \sum_{i=1}^{k-1}\alpha_i\left\langle \bm{g}_i,\bm{g}_0 \right\rangle - \frac{A_k}{2}\|\bm{g}_0\|^2 \nonumber
\\ & =\frac{1}{2}\|\bm{x}_1-\bm{x}^*\|^2 + A_k(f_0 -f^*) - \left\langle \bm{x}_1-\bm{x}_k ,\bm{x}_0-\bm{x}^* \right \rangle +  \sum_{i=1}^{k-1}\alpha_i\left\langle \bm{g}_i,\bm{g}_0 \right\rangle - \frac{A_k}{2}\|\bm{g}_0\|^2 .
\label{eq:Ak-ub1}
\end{align}

To continue, we make the observation that
\begin{align}
%& A_k (f_k-f^*) -A_k(f_0-f^*) = A_k(f_k-f_0) ;\nonumber
%\\ 
\frac{1}{2}\|\bm{x}_1-\bm{x}^*\|^2 - \frac{1}{2}\|\bm{x}_k-\bm{x}^*\|^2  -\left\langle \bm{x}_1-\bm{x}_k ,\bm{x}_0-\bm{x}^* \right \rangle 
%\nonumber\\ 
&  
%\quad 
=\frac{1}{2}\|\bm{x}_1\|^2  - \frac{1}{2}\|\bm{x}_k\|^2 - \left\langle \bm{x}_1,\bm{x}_0  \right\rangle + \left\langle \bm{x}_k,\bm{x}_0  \right\rangle \nonumber
\\ &  = \frac{1}{2}\|\bm{x}_1-\bm{x}_0\|^2 - \frac{1}{2}\|\bm{x}_k-\bm{x}_0\|^2,\nonumber
\end{align}
which combined with \eqref{eq:Ak-ub1} yields
\begin{align}
A_k(f_k-f_0) +C_k \|\bm{g}_k\|^2  + \frac{1}{2}\|\bm{x}_k-\bm{x}^*\|^2 \leq  \frac{1}{2}\|\bm{x}_1-\bm{x}_0\|^2 +\sum_{i=1}^{k-1}\alpha_i\left\langle \bm{g}_i,\bm{g}_0 \right\rangle - \frac{A_k}{2}\|\bm{g}_0\|^2\nonumber
\end{align}
as claimed.
\end{proof}

Furthermore, the result in Lemma~\ref{lemma:key} allows us to control the gradient norm at the last step, provided that a primitive stepsize schedule is adopted. 
%, as asserted by the following lemma. 
%
\begin{lemma}\label{lemma:key3} 
Assume $\bm{s}=\bm{\alpha}_{1:k-1}$ is a primitive stepsize schedule. 
Assume  $\bm{x}_1 = \bm{x}_0 - \alpha_0 \bm{g}_0$. Then one has
\begin{align}
  C_k\|\bm{g}_k\|^2  &\leq \left(\frac{\alpha_0^2}{2}+\frac{(A_k+1)^2}{2}-\alpha_0 -\frac{A_k}{2} \right) \|\bm{g}_0\|^2;\label{eq:we1}
 \\f_k -f_0 &\leq \frac{1}{A_k} \left( \frac{1}{2}\alpha_0^2-\frac{A_k}{2}-\alpha_0+\frac{1}{2}\right)\|\bm{g}_0\|^2.\label{eq:we2}
\end{align}
\end{lemma}

\begin{proof} 
%By Lemma~\ref{lemma:basic_property}, $\overline{\bm{s}}_i$ is a primitive stepsize schedule.
  Because $\bm{\alpha}_{1:k-1}$ is a primitive stepsize schedule,  it follows from Lemma~\ref{lemma:key}  that 
\begin{align}
A_{k}(f_{k}-f_0) + \frac{1}{2}\|\bm{x}_k-\bm{x}_0\|^2 +C_k \|\bm{g}_{k}\|^2 \leq \frac{1}{2}\|\bm{x}_1-\bm{x}_0\|^2 +\sum_{i=1}^{k-1}\alpha_i \left \langle \bm{g}_i, \bm{g}_0 \right\rangle - \frac{A_k}{2}\|\bm{g}_0\|^2.
\label{eq:13579}
\end{align}
We also make note of the following basic facts:
\begin{subequations}
\begin{align}
  \bm{x}_1 &=  \bm{x}_0 - \alpha_0 \bm{g}_0; \label{eq:properties1}
 \\ \sum_{i=1}^{k-1}\alpha_i \bm{g}_i &= \bm{x}_1 -\bm{x}_k = \bm{x}_0 - \bm{x}_k - \alpha_0 \bm{g}_0;
 \label{eq:properties2}
  \\ f_0 -f_k &\leq \left\langle \bm{g}_0, \bm{x}_0-\bm{x}_k\right\rangle - \frac{1}{2}\|\bm{g}_0-\bm{g}_k\|^2 ;
   \\ (A_k+1)\left\langle \bm{g}_0, \bm{x}_0-\bm{x}_k\right\rangle &\leq \frac{1}{2}\|\bm{x}_k-\bm{x}_0\|^2 + \frac{(A_k+1)^2}{2}\|\bm{g}_0\|^2.
\end{align}
\end{subequations}
Putting the above inequalities together, we arrive at
\begin{align}
C_k\|\bm{g}_k\|^2 \leq \left(\frac{\alpha_0^2}{2}+\frac{(A_k+1)^2}{2}-\alpha_0 -\frac{A_k}{2} \right) \|\bm{g}_0\|^2.\nonumber
\end{align}
 \eqref{eq:we1} is proven.

 To prove \eqref{eq:we2}, it suffices to note from \eqref{eq:13579} that 
\begin{align}
A_k(f_k-f_0) + \frac{1}{2}\|\bm{x}_k-\bm{x}_0\|^2 
&\leq \frac{1}{2}\|\bm{x}_1-\bm{x}_0\|^2 +\sum_{i=1}^{k-1}\alpha_i \left \langle \bm{g}_i, \bm{g}_0 \right\rangle - \frac{A_k}{2}\|\bm{g}_0\|^2 \notag\\
&= \frac{\alpha_0^2}{2}\|\bm{g}_0\|^2 +
 \left \langle \bm{x}_0 - \bm{x}_k - \alpha_0 \bm{g}_0, \bm{g}_0 \right\rangle 
- \frac{A_k}{2}\|\bm{g}_0\|^2 \notag\\
& = \left(\frac{1}{2}\alpha_0^2 - \frac{A_k}{2}-\alpha_0 \right)\|\bm{g}_0\|^2 + \left\langle \bm{x}_0-\bm{x}_k  ,\bm{g}_0 \right\rangle \nonumber
\\ & \leq \left(\frac{1}{2}\alpha_0^2 - \frac{A_k}{2}-\alpha_0 +\frac{1}{2}\right)\|\bm{g}_0\|^2  + \frac{1}{2}\|\bm{x}_k-\bm{x}_0\|^2,
\end{align}
where the second line makes use of \eqref{eq:properties1} and \eqref{eq:properties2}, 
and  the last line results from the elementary inequality $2\langle \bm{a}, \bm{b}\rangle \leq \|\bm{a}\|^2+\|\bm{b}\|^2$. 
This concludes the proof.
\end{proof}

Additionally, the following lemma enables effective control of the gradient norms in all intermediate steps. 

\begin{lemma}\label{lemma:refine1}
%[Refined version of Lemma~\ref{lemma:s1}] 
Consider $i\geq 1$ and $\alpha\geq 0$, and let $k= 2^i$. 
Denote by $\overline{\bm{s}}_{i}= [\alpha_1,\dots,\alpha_{k-1}]^{\top}$ the $i$-th order silver stepsize schedule. 
%Recall $\overline{\bm{s}}_i = [\alpha_1,\alpha_2,\ldots, \alpha_{k-1}]^{\top}$  is the $i$-th order silver stepsize schedule. 
Fix $\bm{x}_0$, set $\alpha_0  = \alpha$ and let $\bm{x}_1 = \bm{x}_0 - \alpha_0 \bm{g}_0$. If $\alpha \geq (\sqrt{2}-1)A_k+\sqrt{2}$, then one has 
\begin{align}
 f_{\ell}-f_0 \leq  432 \alpha^2 \|\bm{g}_0\|^2 \nonumber
\end{align}
for any $\ell$ obeying $1\leq \ell\leq k-1$.
\end{lemma}
\begin{proof}
%[Proof of Lemma~\ref{lemma:refine1}]
%
Denote by $\widehat{A}_j = \rho^{j}-1$ the aggregate stepsize of the $j$-th order silver stepsize schedule for $j\geq 1$ (see, e.g., Lemma~\ref{lemma:basic_property}).
Consider any $\ell \in [1,k-1]$, then there exist $1\leq p \leq i$ and $i> m_1> m_2> ...,> m_{p}\geq 0$ such that  $$\ell = \sum_{j=1}^p 2^{m_j}.$$ 
Also, take
$$
\tau_{0}=0 \qquad \text{and}\qquad \tau_{j} = \sum_{j'=1}^{j}2^{m_j},$$ 
and hence $\tau_p = \ell$.

Now, consider the stepsize schedule $\bm{\alpha}_{\tau_j:\tau_{j+1}-1} = [\alpha_i]_{\tau_j\leq i< \tau_{j+1}}$, whose length is $\tau_{j+1}-\tau_j = 2^{m_{j+1}}$. By construction,  we know that $\bm{\alpha}_{\tau_j +1:\tau_{j+1}-1}$ corresponds to the $m_{j+1}$-th order silver stepsize schedule, and  $$\alpha_{\tau_{j+1}} = (\sqrt{2}-1)\widehat{A}_{m_{j+1}}+\sqrt{2}$$ for all $j$. Combining this with the fact $\widehat{A}_j=\rho^j-1$ and the assumption that $\alpha \geq (\sqrt{2}-1)A_k + \sqrt{2}$ (recall that $\alpha=\alpha_0$)  yields 
\begin{align}
\alpha_{\tau_{j+1}}\leq \alpha_{\tau_j}/2,
\label{eq:alpha-half}
\end{align}
provided that $j\geq 0$ and $m_{j+1}\geq 2$.

Invoking Lemma~\ref{lemma:key3} with this stepsize schedule, we can demonstrate that
\begin{align}
\widehat{A}_{m_{j+1}}(\widehat{A}_{m_{j+1}}+1)\|\bm{g}_{\tau_{j+1}}\|^2 &\leq \Big( \alpha_{\tau_j}^2 -2 \alpha_{\tau_j} +  \widehat{A}_{m_{j+1}}^2 +  \widehat{A}_{m_{j+1}}  +1  \Big) \|\bm{g}_{\tau_j}\|^2;\label{eq:rkw1}
\\ f_{\tau_{j+1}}-f_{\tau_j} &\leq \frac{1}{\widehat{A}_{m_{j+1}}}\left( \frac{1}{2} \alpha_{\tau_j}^2  - \frac{ \widehat{A}_{m_{j+1}}}{2} -\alpha_{\tau_j} + 1\right) \|\bm{g}_{\tau_j}\|^2 \leq \frac{1}{2} \alpha_{\tau_j}^2 \|\bm{g}_{\tau_j}\|^2 .\label{eq:rkw2}
\end{align}
It then follows from \eqref{eq:rkw1} that 
\begin{align}
\|\bm{g}_{\tau_{j+1}}\|^2 &\leq 
\frac{\leq \big( \alpha_{\tau_j}^2  +  \widehat{A}_{m_{j+1}}^2 +  \widehat{A}_{m_{j+1}}  +1  \big) \|\bm{g}_{\tau_j}\|^2}{\widehat{A}_{m_{j+1}}(\widehat{A}_{m_{j+1}}+1)} \notag\\
&\leq \left(   \frac{\alpha_{\tau_j}^2}{\widehat{A}_{m_{j+1}}(\widehat{A}_{m_{j+1}}+1)}  + 1\right)\|\bm{g}_{\tau_j}\|^2,\label{eq:rkw3}
\end{align}
which we would like further control by dividing into two cases. 
\begin{itemize}
    \item {\em Case 1: $m_{j+1}\geq 2$.} In this case, we have  $\widehat{A}_{m_{j+1}}\geq \rho^2-1 = 2+ 2\sqrt{2}$. Observing that $\alpha_{\tau_{j+1}} = (\sqrt{2}-1)\widehat{A}_{m_{j+1}}+\sqrt{2}$ by construction, one can easily verify that $$\widehat{A}_{m_{j+1}}(\widehat{A}_{m_{j+1}}+1)\geq (\sqrt{2}+1)\alpha_{\tau_{j+1}}^2,$$ 
    which combined with \eqref{eq:rkw3} implies that
\begin{align}
\|\bm{g}_{\tau_{j+1}}\|^2 \leq \left(   \frac{\alpha_{\tau_j}^2}{\alpha_{\tau_{j+1}}^2}\cdot (\sqrt{2}-1)  + 1\right)\|\bm{g}_{\tau_j}\|^2.
\end{align}
This taken together with the property $\alpha_{\tau_{j+1}}\leq \alpha_{\tau_j}/2$ (cf.~\eqref{eq:alpha-half}) leads to
\begin{align}
\alpha_{\tau_{j+1}}^2 \|\bm{g}_{\tau_{j+1}}\|^2 \leq \left(\sqrt{2}-\frac{3}{4} \right)\alpha_{\tau_{j}}^2 \|\bm{g}_{\tau_j}\|^2.\label{eq:cxxxq1}
\end{align}

    \item  {\em Case 2: $m_{j+1}<2$.}
In this case, 
it is readily seen from \eqref{eq:rkw3} that 
\begin{align}
 \|\bm{g}_{\tau_{j+1}}\|^2 \leq \left(   \frac{\alpha_{\tau_j}^2}{\widehat{A}_{m_{j+1}}(\widehat{A}_{m_{j+1}}+1)}  + 1\right)\|\bm{g}_{\tau_j}\|^2 \leq \alpha_{\tau_j}^2 \|\bm{g}_{\tau_j}\|^2 .\nonumber
\end{align}
Moreover, we make the observation that
$$\alpha_{\tau_{j+1}} = (\sqrt{2}-1)\widehat{A}_{m_{j+1}}+\sqrt{2}\leq 
(\sqrt{2}-1) (\rho-1) +\sqrt{2}
=2,
$$ 
which allows us to reach
\begin{align}
\alpha_{\tau_{j+1}}^2  \|\bm{g}_{\tau_{j+1}}\|^2 \leq  12\alpha_{\tau_j}^2 \|\bm{g}_{\tau_j}\|^2.\label{eq:cxxxq2}
\end{align}
\end{itemize}

Putting \eqref{eq:cxxxq1} and \eqref{eq:cxxxq2} together, we can conclude that for any $j\geq 1$,
\begin{align}
\alpha_{\tau_j}^2 \|\bm{g}_{\tau_{j}}\|^2 \leq 432 \left(\sqrt{2}-\frac{3}{4}\right)^j \alpha_{\tau_0}^2 \|\bm{g}_{\tau_0}\|^2 =  432 \left(\sqrt{2}-\frac{3}{4}\right)^j \alpha^2 \|\bm{g}_{0}\|^2 . 
\end{align}
This taken collectively with \eqref{eq:rkw2} gives
\begin{align}
\notag f_{\ell} -f_0 &= \sum_{j=0}^{p-1} (f_{\tau_{j+1}}-f_{\tau_j}) \leq \frac{1}{2}\sum_{j=0}^{p-1} \alpha_{\tau_j}^2 \|\bm{g}_{\tau_j}\|^2 \\ &\leq \left(\frac{1}{2}+ 216\sum_{j\geq 1}\left(\sqrt{2}-\frac{3}{4}\right)^j \right)\alpha^2 \|\bm{g}_0\|^2\leq 432\alpha^2 \|\bm{g}_0\|^2
\end{align}
as claimed.
\end{proof}

\subsection{Proof of Theorem~\ref{thm:main}}\label{sec:proof}

We are now positioned to prove our main result in Theorem~\ref{thm:main}, 
based on the stepsize schedule $\widehat{\bm{s}}=[\alpha_i]_{i=1}^{\infty}$ constructed in Section~\ref{sec:const}. 
Let us remind the readers of several notation below. 
\begin{itemize}

    \item $t_i$: the length of the $i$-th subsequence $\widehat{\bm{s}}_i=[\alpha_1,\dots,\alpha_{t_i}]^{\top}$ (see Section~\ref{sec:const}), corresponding to the first $t_i$ stepsizes in $\widehat{\bm{s}}$. .

    \item $o_t$: the integer such that  $ \sum_{j=1}^{o_t-1} k_j 2^j <t\leq \sum_{j=1}^{o_t} k_j 2^j $. Clearly, the length of the $(i+1)$-th subsequence  (including the $(t_i+1)$-th step) is $2^{o_{t_i+1}}$,  and $t_{i+1}=t_i + 2^{o_{t_i+1}}\leq 3t_i$.

    \item $A_t$ and $C_t$: 
    $A_t = \sum_{i=1}^{t-1}\alpha_{i}$ and $C_t =  \frac{A_t (A_t+1)}{2}$, where we suppress the dependency on $\widehat{\bm{s}}$ for notational convenience.  

    \item $\alpha_j(\overline{s}_i)$: the $j$-th stepsize in the sequence $\overline{\bm{s}}_i$. 
    
    %$A_t(\overline{s})$ and $C_t(\overline{s})$: $A_t = \sum_{i=1}^{t-1}\alpha_{i}$ and $C_t =  \frac{A_t (A_t+1)}{2}$, where we suppress the dependency on $\widehat{\bm{s}}$ for notational convenience. 
    
\end{itemize}
It is also worth noting that  Lemma~\ref{lemma:sequence_computation} gives
\begin{equation}
2^{o_{t_i+1}} \leq 2 \cdot 2^{o_{t_i}} \leq 4 t_i^{\frac{1}{c+1}}.
\label{eq:2-ot-UB}
\end{equation}
%

%Here and below, 
%we use $o_t$ to denote the integer such that  $ \sum_{j=1}^{o_t-1} k_j 2^j <t\leq \sum_{j=1}^{o_t} k_j 2^j $; and for any given $i\geq 1$, we denote by  $t_i$ the length of the $i$-th subsequence $\widehat{\bm{s}}_i$ (see Section~\ref{sec:const}). 
%Note that the length of the $(i+1)$-th subsequence  (including the $(t_i+1)$-th step) is $2^{o_{t_i+1}}$,  and $t_{i+1}=t_i + 2^{o_{t_i+1}}$.  

Consider any $i\geq 1$. 
%Denote by $\widehat{\bm{s}}_i=[\alpha_1,\dots,\alpha_{t_i}]^{\top}$ the first $t_i$ stepsizes in $\widehat{\bm{s}}$. 
%So $A_{t_i+1}(\widehat{\bm{s}}_i) = A_{t_i+1}(\widehat{\bm{s}}) = A_{t_i+1}$.
 %
 In view of Lemma~\ref{lemma:well-define}, we know that $\widehat{\bm{s}}_i$ is  primitive. 
 % \yxc{add explanation here} \textcolor{green}{[Zihan: please check.]} 
 Given that  $\{\bm{x}_{j}\}_{j=1}^{t_i+1}$ is the GD trajectory  with stepsize schedule $\widehat{\bm{s}}_i$, we see from Definition~\ref{defn:primitive} of the primitive stepsize schedule that
 \begin{align}
A_{t_i+1}(f_{t_i+1}-f^*) + C_{t_i+1}\|\bm{g}_{t_i+1}\|^2 +\frac{1}{2}\|\bm{x}_{t_i+1}-\bm{x}^*\|^2 \leq \frac{1}{2}\|\bm{x}_1-\bm{x}^*\|^2 ,\nonumber
 \end{align}
which immediately implies that 
 \begin{subequations}
\begin{align}
      f_{t_{i}+1} -f^*  &\leq \frac{\|\bm{x}_1-\bm{x}^*\|^2}{A_{t_i+1}};\label{eq:krr0}
  \\  \|\bm{g}_{t_i+1}\|^2 &\leq \frac{\|\bm{x}_1-\bm{x}^*\|^2}{2C_{t_i+1}} \leq  \frac{\|\bm{x}_1-\bm{x}^*\|^2}{A_{t_i+1}^2} .\label{eq:krr1}
\end{align}
\end{subequations}
%
%Recall that $\overline{\bm{s}}_m=\alpha_{1:2^{m}-1}(\overline{\bm{s}}_m)$ for $m\geq  1$.

Additionally, by construction we have $\alpha_{t_i+1} = \varphi(x,y)$ due to the concatenation operation, where 
\begin{align}
 x &= 
 %A_{t_i+1}= 
 \sum_{j=1}^{t_i}\alpha_j \geq \frac{1}{36}t_i^{\frac{c+\log_2{\rho}}{c+1}} \qquad \qquad \quad \qquad 
% (\text{by Lemma~\ref{lemma:sequence_computation}})
 ;
 \nonumber
\\  y & 
%= A_{t_{i+1}+1} - A_{t_i+2}  
=\sum_{j=t_i+2}^{t_{i+1}}\alpha_j  
%\nonumber\\ &  
=\sum_{j=1}^{2^{o_{t_i+1}}-1}\alpha_{j}(\overline{\bm{s}}_{o_{t_i+1}}) = \rho^{o_{t_i+1}}-1\leq 2\cdot t_i^{\frac{\log_2{\rho}}{c+1}}%(\text{by Lemma~\ref{lemma:sequence_computation}})
.
\nonumber
\end{align}
Here, both of the inequalities above arise from Lemma~\ref{lemma:sequence_computation}. It is also easy to observe that $x\geq y$.
It then follows that
\begin{align}
\alpha_{t_i+1}=\varphi(x,y)& = \frac{-(x+y) + \sqrt{ (x+y+2)^2 + 4(x+1)(y+1)   }}{2} \nonumber
\\ & = \frac{ 4( xy+2x+2y+2 ) }{2 (x+ y  +\sqrt{  (x+y+2)^2 +4(x+1)(y+1) } ) } \nonumber
\\ & \leq  \frac{xy + 2x+2y+2}{x+y+2}\nonumber
\\ & \leq y + 2\nonumber
\\ & = \rho^{o_{t_i+1}}+1.\nonumber
\end{align}
Moreover, recognizing that
\begin{align}
\frac{\partial \varphi(x,y)}{\partial x} = \frac{1}{2}\left(-1+ \frac{x+3y+4}{\sqrt{   x^2 + (6y+8)x + 8y + 8}} \right)\geq 0\nonumber
\end{align}
for all $(x,y)\geq 0$, 
we immediately obtain 
\begin{align}
\alpha_{t_i+1}= \varphi(x,y)& 
\geq \varphi(y,y)=
%= \frac{-(x+y) + \sqrt{ (x+y+2)^2 + 4(x+1)(y+1)   }}{2} \nonumber
%\\  & \geq \varphi(y,y)  = 
(\sqrt{2}-1)y+\sqrt{2}.\nonumber
\end{align}
%Here the last inequality is by the fact that .
%
Invoking 
%\eqref{eq:krr1} and  
Lemma~\ref{lemma:refine1} over the $(i+1)$-th sub-sequence with $\alpha = \alpha_{t_i+1}\geq (\sqrt{2}-1)y +\sqrt{2} $, %and noting that  $\alpha_{t_i+1}\leq \rho^{o_{t_i+1}}+3 = O\big(t^{\frac{\log_2\rho}{c+1}}\big)$, 
we can show, for any $\ell$ obeying $ t_i+1<\ell \leq t_{i+1}$, that
\begin{align}
f_{\ell}-f_{t_i+1}\leq 432 \alpha_{t_i+1}^2 \|\bm{g}_{t_i+1}\|^2 & \overset{\mathrm{(i)}}{\leq} O\left( \frac{\alpha_{t_i+1}^2}{A_{t_i+1}^2} \|\bm{x}_1-\bm{x}^*\|^2\right)\nonumber
\\ & \overset{\mathrm{(ii)}}{\leq} O\left(\frac{\|\bm{x}_1-\bm{x}^*\|^2   t_i^{\frac{2\log_2\rho}{c+1}}}{t_i^{\frac{2(c+\log_2\rho)}{c+1}}} \right)\nonumber
\\&  = O\left(\frac{\|\bm{x}_1-\bm{x}^*\|^2}{\ell^{\frac{2(c+\log_2\rho) -2\log_2\rho}{c+1}}} \right)\nonumber
\\ & = O\left( \frac{\|\bm{x}_1-\bm{x}^*\|^2}{\ell^{\frac{2c}{c+1}}}\right).\nonumber
\end{align}
Here, (i) arises from \eqref{eq:krr1}, whereas (ii) invokes Lemma~\ref{lemma:sequence_computation}, inequality \eqref{eq:2-ot-UB}, as well as the property that 
$$
\alpha_{t_i+1}\leq \rho^{o_{t_i+1}}+3 = O(\rho^{o_{t_i+1}})
\leq O\big({t_i}^{\frac{\log_2\rho}{c+1}} \big).
$$
%}

This taken together with \eqref{eq:krr0} further results in
\begin{align}
f_{\ell}-f^* 
&= f_{\ell}-
f_{t_i+1} + \big( f_{t_i+1} - f^* \big)
\notag \\
&\leq  O\left(  \frac{\|\bm{x}_1-\bm{x}^*\|^2}{\ell^{\frac{    c+\log_2\rho   }{c+1}}} + \frac{\|\bm{x}_1-\bm{x}^*\|^2}{A_{t_i+1}}\right) =  O\left(  \frac{\|\bm{x}_1-\bm{x}^*\|^2}{\ell^{\frac{    c+\log_2\rho   }{c+1}}}\right).\nonumber
\end{align}
%
%By virtue of  \eqref{eq:krr0} and Lemma~\ref{lemma:sequence_computation}, 
Consequently, 
we have shown that, for any $\ell\in \cup_{i\geq 1} (t_i, t_{i+1}]  = [3, \infty)$, 
\begin{align}
f_{\ell}-f^* &\leq O\left( \frac{\|\bm{x}_1-\bm{x}^*\|^2}{ A_{t_i+1}}+ \frac{\|\bm{x}_1-\bm{x}^*\|^2}{\ell^{\frac{    2c  }{c+1}}}\right) \leq  O\left(  \frac{\|\bm{x}_1-\bm{x}^*\|^2}{\ell^{\frac{2c}{c+1}}} +  \frac{\|\bm{x}_1-\bm{x}^*\|^2}{\ell^{\frac{    c+\log_2\rho   }{c+1}}}\right) \\
&= O\left(    \frac{\|\bm{x}_1-\bm{x}^*\|^2}{\ell^{\frac{    2\log_2\rho   }{1+\log_2\rho}}}\right),
\end{align}
where the last line follows by taking $c = \log_2\rho$. 

It remains to justify the advertised result when $\ell<3$. Towards this end, it is easily seen that 
\begin{align*}
f_1 - f^* &\leq \frac{\|\bm{x}_1-\bm{x}^*\|^2}{2} 
\qquad \text{and} \\
f_2-f^* &\leq f_1-f^* + \frac{\alpha_1^2 +2\alpha_1}{2}\|\bm{g}_1\|^2 \leq  (1+\alpha_1^2 +2\alpha_1)(f_1-f^*)\leq \frac{9\|\bm{x}_1-\bm{x}^*\|^2}{2},
\end{align*}
where we have made use of \eqref{eq:local_error}. 
We have thus completed the proof. 

\section{Extension to smooth and strongly convex problems}\label{sec:strong_cvx}

%\yxc{Let's post the convex part first and work on strongly convex one later on.}

In this section, we 
further extend our result to accommodate smooth and strongly convex optimization; that is, we assume that the objective function $f$ in \eqref{eq:convex-smooth} is $1$-smooth and $\mu$-strongly convex for some strong convexity parameter $\mu\in (0,1]$. Here and throughout, we denote by $\kappa=1/\mu$ the condition number. Our result, which guarantees acceleration of standard GD theory (i.e., $\exp(-\Omega(T/\kappa))$ in an anytime manner, is stated as follows. 
\begin{theorem}\label{thm:strcox}
There is a stepsize schedule $\{\alpha_t\}_{t=1}^{\infty}$,  generated without knowing the stopping time, such that the gradient descent iterates \eqref{eq:GD} obey
\begin{align}
f(\bm{x}_T)-f^* \leq O\left(  \exp\left(-\frac{CT}{\kappa^{\varsigma}}\right) \|\bm{x}_1-\bm{x}^*\|^2\right), \label{eq:ssrt}
\end{align}
where 
%$\kappa = 1/{\mu}$,  
$\varsigma =1/\vartheta =  \frac{1+\log_2\rho}{2\log_2\rho}<0.893$, and $C>0$ is some numerical constant. Here, $T$ denotes an arbitrary stopping time that is unknown a priori. 
\end{theorem}
\begin{proof}[Proof of Theorem~\ref{thm:strcox}]
%Before proceeding, we first note that the $\mu$-strong convex of $f$ gives 
%for any $\bm{x}$,
%\begin{align}
%f(\bm{x}) - f(\bm{x}^*)\geq \frac{\mu}{2}\|\bm{x}-%\bm{x}^*\|^2
%\qquad \text{for any }\bm{x}.\nonumber
%\end{align}
%
%Also, 
Recall our construction of $\widehat{\bm{s}}$ in the proof of Theorem~\ref{thm:main} (see Section~\ref{sec:const}). According to Theorem~\ref{thm:main}, there exists a universal constant $C_0>0$ such that running GD with the stepsize schedule $\widehat{\bm{s}}$ achieves
\begin{align}
f(\bm{x}_t)-f^* \leq \frac{C_0 \|\bm{x}_1-\bm{x}^*\|^2}{t^{\vartheta}}.\nonumber
\end{align}

Let us begin by constructing a stepsize schedule tailored to the $\mu$-strongly convex problem. 
Take $\tau=\tau(\mu)$ to be the smallest integer such that $A_{\tau+1}(\widehat{\bm{s}})\geq \frac{4C_0}{\mu} = 4C_0\kappa$. Lemma~\ref{lemma:sequence_computation} tells us that 
\begin{equation}
\frac{1}{36}\tau^{\vartheta}=\frac{1}{36}\tau^{\frac{c+\log_2 \rho}{c+1}}\leq A_{\tau}(\widehat{\bm{s}})\leq 4C_0\kappa,
\label{eq:Atau-UB}
\end{equation}
which implies that $$\tau\leq 144C_0\kappa^{\frac{1}{\vartheta}} = 144C_0\kappa^{\varsigma}.$$
Now, let $\widetilde{\bm{s}} = \bm{\alpha}_{1:\tau}(\widehat{\bm{s}})$ (i.e., the first $\tau$ stepsizes in $\widehat{\bm{s}}$), and set $\widetilde{\bm{s}}^*$ to be the infinite stepsize schedule $[ \widetilde{\bm{s}}^{\top},\widetilde{\bm{s}}^{\top},\ldots ]^{\top}$; that is, $\alpha_{i\tau+j}(\widetilde{\bm{s}}^*) =\alpha_j(\widetilde{\bm{s}}) = \alpha_j(\widehat{\bm{s}})$ for any $i\geq 0$ and $1\leq j \leq \tau$.

Next, we would like to show that the claimed result \eqref{eq:ssrt} holds with the stepsize schedule   $\widehat{\bm{s}}^*$.
In view of Theorem~\ref{thm:main}, we know that
\begin{align}
  f_{j}-f^* &\leq   \frac{C_0\|\bm{x}_1-\bm{x}^*\|^2}{j^{\vartheta}} \leq 55C_0 \exp\left( -\frac{j}{36C_0\kappa^{\varsigma}} \right)\cdot \|\bm{x}_1-\bm{x}^*\|^2\qquad \text{for all } 1\leq j\leq \tau;  \label{eq:xy0}
\\  f_{\tau+1}-f^* &\leq \frac{C_0\|\bm{x}_1-\bm{x}^*\|^2}{  \tau^{\theta} } = \frac{\|\bm{x}_1-\bm{x}^*\|^2}{144 \kappa^{\varsigma\cdot \vartheta}} = \frac{\|\bm{x}_1-\bm{x}^*\|^2}{144\kappa} = \frac{\mu \|\bm{x}_1-\bm{x}^*\|^2}{144},
\end{align}
where we have invoked \eqref{eq:Atau-UB}. 
Observing that 
$
f_{\tau+1}-f^* \geq \frac{\mu \|\bm{x}_{\tau+1}-\bm{x}^*\|^2}{2}$ due to $\mu$-strong convexity, we have $$\|\bm{x}_{\tau+1}-\bm{x}^*\|^2 \leq \frac{1}{72}\|\bm{x}_1-\bm{x}^*\|^2.$$ 

Invoking similar arguments reveals that: for any $i\geq 1$ and $1\leq j \leq \tau$, one has
\begin{align}
\frac{\mu}{2}\|\bm{x}_{i\tau +1}-\bm{x}^*\|^2  \leq f_{i\tau +1}-f^* 
%\leq \frac{C_0 \|\bm{x}_{(i-1)\tau+1}-\bm{x}^*\|^2}{A_{\tau+1}(\widehat{\bm{s}})}
\leq \frac{\mu \|\bm{x}_{(i-1)\tau+1}-\bm{x}^*\|^2}{4}\nonumber
\end{align}
\begin{align}
\text{and}\qquad f_{i\tau + j}-f^* 
\leq \frac{C_0 \|\bm{x}_{i\tau+1}-\bm{x}^*\|^2}{j^{\vartheta}}
\leq C_0 \|\bm{x}_{i\tau+1}-\bm{x}^*\|^2.\nonumber
\end{align}
As a result, we can deduce that 
\begin{align}
\|\bm{x}_{i\tau +1}-\bm{x}^*\|^2  & \leq \frac{1}{2} \|\bm{x}_{(i-1)\tau+1}-\bm{x}^*\|^2\nonumber
\\ & \leq \left( \frac{1}{2}\right)^i \|\bm{x}_1-\bm{x}^*\|^2 \nonumber
\\ & \leq \exp\left( -\log 2 \cdot  \frac{i\tau +1}{2\tau} \right) \|\bm{x}_1-\bm{x}^*\|^2 \nonumber
\\ & \leq \exp\left( - \frac{i\tau + 1}{576C_0 \kappa^{\varsigma}} \right)\|\bm{x}_1-\bm{x}^*\|^2,\nonumber
\end{align}
and as a result, 
\begin{align}
f_{i\tau + j}-f^* 
\leq C_0 \|\bm{x}_{i\tau+1}-\bm{x}^*\|^2
\leq C_0 \exp\left( - \frac{i\tau + 1}{576C_0 \kappa^{\varsigma}} \right)\|\bm{x}_1-\bm{x}^*\|^2\leq C_0\exp\left( - \frac{i\tau + j}{1152C_0 \kappa^{\varsigma}} \right)\|\bm{x}_1-\bm{x}^*\|^2.\label{eq:xy1}
\end{align}

To finish up, combine \eqref{eq:xy0} and \eqref{eq:xy1} to arrive at
%: %for any $T\geq 1$, it holds that 
\begin{align}
f_T-f^* 
\leq 55C_0\exp\left(-\frac{CT}{\kappa^{\varsigma}} \right)\cdot \|\bm{x}_1-\bm{x}^*\|^2
\qquad \text{for any }T\geq 1,\nonumber
\end{align}
where $C = \frac{1}{1152C_0}$. This concludes   
the proof.
\end{proof}

\section*{Acknowledgements}

YC is supported in part by the Alfred P.~Sloan Research Fellowship, the AFOSR grant FA9550-22-1-0198, the ONR grant N00014-22-1-2354,  and the NSF grant CCF-2221009. 
SSD is supported in part by the Alfred P.~Sloan Research Fellowship, NSF DMS 2134106, NSF CCF 2212261, NSF IIS 2143493, and NSF IIS 2229881. 
JDL acknowledges support of  NSF CCF 2002272, NSF IIS 2107304,  NSF CIF 2212262, ONR Young Investigator Award, and NSF CAREER Award 2144994. We would like to thank Ernest Ryu for helpful discussions and references regarding the literature.

\bibliographystyle{apalike}
\bibliography{bib-GD.bib}

%\newpage
\appendix

\section{Proof of Lemma~\ref{lemma:sequence_computation}}
\label{sec:proof-lem-sequence-computation}
First, let us look at the 
case with $o_t=1$, 
for which we have  $t\leq 2k_1= 2\cdot \left\lfloor2^{c+1}\right\rfloor$. 
Given that $\varphi(x,y) >1$ for all $x,y \geq 0$, we can easily verify that
$$A_{t+1}(\widehat{\bm{s}})\geq t\geq  \frac{1}{36} t^{\frac{c+\log_2\rho}{c+1}}.$$ 
It is also easily seen that 
$
    2^{o_t}=2\leq 2t^{\frac{1}{c+1}}.
$

Now, let us turn to the case 
where $o_{t}\geq 2$. Let $m \in [1,k_{o_t}]$ be the integer such that $\sum_{j=1}^{o_t-1}k_j 2^j + (m-1) \cdot 2^{o_t} <t\leq  \sum_{j=1}^{o_t-1}k_j 2^j + m \cdot 2^{o_t}$. By definition, we have 
\begin{align}
 t &\leq \sum_{j=1}^{o_t-1}k_j 2^j + m  2^{o_t} \leq  4 \cdot 2^{(c+1)(o_t-1)}+m 2^{o_t};\nonumber
\\  A_{t+1}(\widehat{\bm{s}}) &\geq \sum_{j=1}^{o_t-1}(\rho^j -1)\cdot k_j + (m-1) (\rho^{o_t}-1)\geq \frac{1}{2} \cdot 2^{(c+\log_2\rho) (o_t-1)}+\frac{m-1}{2} \rho^{o_t} ,\nonumber
\end{align}
where the second line invokes Lemma~\ref{lemma:basic_property}. 
\begin{itemize}
\item 
If $m 2^{o_t}\leq 2^{(c+1)(o_t-1)}$, then we have 
$t\leq 3\cdot  2^{(c+1)(o_t-1)}$, which means that $A_{t+1}(\widehat{\bm{s}})\geq \frac{1}{2}\cdot 2^{(c+\log_2\rho)(o_t-1)}\geq \frac{1}{18} t^{\frac{c+\log_2\rho}{c+1}}$. 

\item If $m 2^{o_t}> 2^{(c+1)(o_t-1)}$ --- i.e, $ 2^{o_t c}\geq m> 2^{o_t c -c-1}\geq 1$ --- then one has
\begin{align}
t^{\frac{c+\log_2\rho}{c+1}}\leq  (4 \cdot 2^{(c+1)(o_t-1)}+m 2^{o_t})^{\frac{c+\log_2\rho}{c+1}}< 9 (m 2^{o_t})^{\frac{c+\log_2\rho}{c+1}}\leq 9 \cdot m \rho^{o_t}\leq 36 \cdot \frac{m-1}{2} \rho^{o_t}\leq 36 A_{t+1}(\widehat{\bm{s}}).\nonumber
\end{align}
\end{itemize}
Putting these two cases together establishes the claim~\eqref{eq:At-lower-bound}.

Regarding the second claim, in the case where $o_t\geq 2$, we have
\begin{align}
t\geq \sum_{j=1}^{o_t-1}k_j 2^j \geq \sum_{j=1}^{o_t-1} 2^{(c+1)j}\geq  2^{(c+1)(o_t-1)},
\end{align}
thus indicating that $2t^{\frac{1}{c+1}}\geq 2^{o_t}$.

\section{Proof of preliminary facts from \citet{zhang2024accelerated}}

\subsection{Proof of Lemma~\ref{lemma:key1}}
\label{sec:proof-lemma:key1}
%
% and Lemma~\ref{lemma:key2}

As mentioned previously, this lemma was established by \cite{zhang2024accelerated}. We present the proof for completeness.

To begin with, we single out the following lemma, originally established by  \citet[Lemma 3.1]{zhang2024accelerated}, that plays a key role in the proof of Lemma~\ref{lemma:key1}. We shall provide a proof in Appendix~\ref{sec:proof-lemma:key2}. 
 \begin{lemma}(\citet[Lemma 3.1]{zhang2024accelerated})\label{lemma:key2}
Assume that $\bm{\alpha}_{1:\ell-1}$ is  primitive. For any $\alpha\in [1, A_{\ell}+2)$, if we set $\alpha_0 =\alpha$, then it holds that
\begin{align}
f_0 -f_{\ell}\geq \frac{A_{\ell} + 3\alpha - 2\alpha^2}{2(A_{\ell}+2- \alpha)} \|\bm{g}_0\|^2 + \frac{2A_{\ell}^2 + 3A_{\ell}+\alpha}{2(A_{\ell}+2-\alpha)} \|\bm{g}_{\ell}\|^2.\nonumber
\end{align}
%For $\alpha = 0$, it holds that $f_0 =f_1 \geq f_{\ell}$.
 \end{lemma}

%\textbf{Lemma 2.(Restatement)} \emph{
%Consider a stepsize schedule $\bm{\alpha}_{1:k-1}$, 
%where both $\bm{\alpha}_{1:\ell-1}$ and $\bm{\alpha}_{\ell+1,k-1}$ are primitive sequences. 
%Define the following function
%
%\begin{equation}
%\varphi(x,y) \coloneqq \frac{-(x+y)+\sqrt{  (x+y+2)^2 +4(x+1)(y+1)
% }}{2}.
%\end{equation}
%
%If $\alpha_{\ell}=\varphi\left(A_{\ell}, A_k-A_{\ell+1}\right)$, then  
%$\bm{\alpha}_{1:k-1}$ is %also a primitive sequence. }

% \begin{proof}[Proof of Lemma~\ref{lemma:key1}]

Next, in view of the definition of the primitive stepsize schedule (cf.~Definition~\ref{defn:primitive}), we can easily see that
\begin{align}
x(f_{\ell}-f^{*})+\frac{x(x+1)}{2}\|\bm{g}_{\ell}\|^{2}+\frac{1}{2}\|\bm{x}_{\ell}-\bm{x}^{*}\|^{2} & \leq\frac{1}{2}\|\bm{x}_{1}-\bm{x}^{*}\|^{2}+\sum_{i=1}^{\ell-1}\alpha_{i}\left(f_{i}-f^{*}-\langle\bm{g}_{i},\bm{x}_{i}-\bm{x}^{*}\rangle+\frac{1}{2}\|\bm{g}_{i}\|^{2}\right),\nonumber\\
y(f_{k}-f^{*})+\frac{y(y+1)}{2}\|\bm{g}_{k}\|^{2}+\frac{1}{2}\|\bm{x}_{k}-\bm{x}^{*}\|^{2} & \leq\frac{1}{2}\|\bm{x}_{\ell+1}-\bm{x}^{*}\|^{2}+\sum_{i=\ell+1}^{k-1}\alpha_{i}\left(f_{i}-f^{*}-\langle\bm{g}_{i},\bm{x}_{i}-\bm{x}^{*}\rangle+\frac{1}{2}\|\bm{g}_{i}\|^{2}\right),\nonumber
\end{align}
where we take $x=A_{\ell}$ and $y=A_{k}-A_{\ell+1}$ for notational simplicity.
	Given that $$z \coloneqq x+y + \alpha = A_{\ell} + (A_{k}-A_{\ell+1}) + \alpha_{\ell}=A_k,$$ Lemma~\ref{lemma:key2} tells us that
\begin{align}
	(x+\alpha) (f_{k}-f_{\ell})\leq -\frac{(x+\alpha)(y+3\alpha -2\alpha^2)}{2(y+2-\alpha)} \|\bm{g}_{\ell}\|^2 - \frac{(x+\alpha)(2y^2 + 3y +\alpha)}{2(y+2-\alpha)} \|\bm{g}_k\|^2.
\end{align}
Adding the above three inequalities and utilizing $z=x+y+\alpha$ yield
 \begin{align}
L_1+L_2 +L_3 +L_4 \leq R_1+R_2+R_3+R_4,\label{eq:L1234-R1234}
 \end{align}
 where 
 \begin{align}
	 L_{1} & =z(f_{k}-f^{*})+\frac{z(z+1)}{2}\|\bm{g}_{k}\|^{2}+\frac{1}{2}\|\bm{x}_{k}-\bm{x}^{*}\|^{2}=A_{k}(f_{k}-f^{*})+C_{k}\|\bm{g}_{k}\|^{2}+\frac{1}{2}\|\bm{x}_{k}-\bm{x}^{*}\|^{2};\nonumber\\
L_{2} & =-\alpha(f_{\ell}-f^{*})+\frac{1}{2}\|\bm{x}_{\ell}-\bm{x}^{*}\|^{2};\nonumber\\
L_{3} & =\frac{x(x+1)}{2}\|\bm{g}_{\ell}\|^{2};\nonumber\\
L_{4} & =\frac{y(y+1)-z(z+1)}{2}\|\bm{g}_{k}\|^{2};\nonumber\\
R_{1} & =\frac{1}{2}\|\bm{x}_{1}-\bm{x}^{*}\|^{2}+\sum_{i=1}^{k-1}\alpha_{i}\bigg(f_{i}-f^{*}-\langle\bm{g}_{i},\bm{x}_{i}-\bm{x}^{*}\rangle+\frac{1}{2}\|\bm{g}_{i}\|^{2}\bigg);\nonumber\\
	 R_{2} & =\frac{1}{2}\|\bm{x}_{\ell+1}-\bm{x}^{*}\|^{2}-\alpha\big(f_{\ell}-f^{*}-\langle\bm{g}_{\ell},\bm{x}_{\ell}-\bm{x}^{*}\rangle\big)-\frac{1}{2}\alpha^{2}\|\bm{g}_{\ell}\|^{2};\nonumber\\
R_{3} & =\left(-\frac{\alpha}{2}+\frac{\alpha^{2}}{2}-\frac{(x+\alpha)(y+3\alpha-2\alpha^{2})}{(y+2-\alpha)}\right)\|\bm{g}_{\ell}\|^{2};\nonumber\\
R_{4} & =-\frac{(x+\alpha)(2y^{2}+3y+\alpha)}{y+2-\alpha}\|\bm{g}_{k}\|^{2}.\nonumber
 \end{align}

We now proceed to simplify \eqref{eq:L1234-R1234}. 
Firstly, 
it is readily seen that 
\begin{align}
L_{2}-R_{2} & =\frac{1}{2}\|\bm{x}_{\ell}-\bm{x}^{*}\|^{2}-\frac{1}{2}\|\bm{x}_{\ell+1}-\bm{x}^{*}\|^{2}-\alpha\langle\bm{g}_{\ell},\bm{x}_{\ell}-\bm{x}^{*}\rangle+\frac{1}{2}\alpha^{2}\|\bm{g}_{\ell}\|^{2}\nonumber\\
 & =\frac{1}{2}\|\bm{x}_{\ell}-\bm{x}^{*}\|^{2}-\frac{1}{2}\|\bm{x}_{\ell}-\bm{x}^{*}-\alpha\bm{g}_{\ell}\|^{2}-\alpha\langle\bm{g}_{\ell},\bm{x}_{\ell}-\bm{x}^{*}\rangle+\frac{1}{2}\alpha^{2}\|\bm{g}_{\ell}\|^{2}\nonumber\\
 & =0. \nonumber
\end{align}
Secondly, recalling our specific choice $\alpha = \varphi(x,y) = \frac{-(x+y)+\sqrt{(x+y+2)^2+4(x+1)(y+1) }}{2}$, we can easily verify that $$\alpha^2 +(x+y)\alpha - (xy+2x+2y+2)=0.$$
This allows one to demonstrate that
\begin{align}
	L_3-R_3&  = \left(\frac{x(x+1)}{2}+\frac{\alpha}{2}-\frac{\alpha^2}{2}+\frac{(x+\alpha)(y+3\alpha -2\alpha^2)}{2(y+2-\alpha)} \right) \|\bm{g}_{\ell}\|^{2} =0;\nonumber
\\ L_4- R_4&  = \left( \frac{y(y+1)-z(z+1)}{2}+\frac{(x+\alpha)(2y^2+3y+\alpha)}{y+2-\alpha}\right) \|\bm{g}_{k}\|^{2} = 0.\nonumber
\end{align}
Substitution into \eqref{eq:L1234-R1234} then results in
$L_1\leq R_1,$ namely, 
\begin{align}
A_{k}(f_{k}-f^{*})+C_{k}\|\bm{g}_{k}\|^{2}+\frac{1}{2}\|\bm{x}_{k}-\bm{x}^{*}\|^{2}
&\leq\frac{1}{2}\|\bm{x}_{1}-\bm{x}^{*}\|^{2}+\sum_{i=1}^{k-1}\alpha_{i}\left(f_{i}-f^{*}-\langle\bm{g}_{i},\bm{x}_{i}-\bm{x}^{*}\rangle+\frac{1}{2}\|\bm{g}_{i}\|^{2}\right),\nonumber\\
&\leq\frac{1}{2}\|\bm{x}_{1}-\bm{x}^{*}\|^{2},
\end{align}
where the last inequality comes from \eqref{eq:fi-fstar-ineq}. 
This completes the proof.

\iffalse
Then we have that
\begin{align}
&z (f_k -f^*) + \frac{z(z+1)}{2}g_k^2 + \frac{1}{2}(x_k-x^*)^2 \nonumber
\\ & = y(f_k -f^*) + \frac{y(y+1)}{2}g_k^2 + \frac{1}{2}(x_k-x^*)^2 + (x+\alpha) (f_k-f^*) + \frac{(x+\alpha)(   x+2y+1+\alpha )}{2}
	\|\bm{g}_{k}\|^{2}. \nonumber
\end{align}
Let $t_1 =y(f_k -f^*) + \frac{y(y+1)}{2}g_k^2 + \frac{1}{2}(x_k-x^*)^2$. For the left terms, we have that
\begin{align}
& (x+\alpha)(f_k -f^*)\nonumber
\\ &  = (x+\alpha)(f_k -f_{\ell}) + x (f_{\ell}-f^*) + \alpha (f_{\ell}-f^* - g_{\ell}(x_{\ell}-x^*)  + \frac{1}{2}\|\bm{g}_{\ell}\|^{2}) + \alpha g_{\ell}(x_{\ell}-x^*) - \frac{\alpha}{2}\|\bm{g}_{\ell}\|^{2}.\nonumber
\\ & \leq 
\end{align}
\fi 
% \end{proof}

 \subsection{Proof of Lemma~\ref{lemma:key2}}
 \label{sec:proof-lemma:key2}
Once again, this lemma has been proven in \citet[Lemma 3.1]{zhang2024accelerated}, and we present the proof for completeness.

According to the definition of the primitive stepsize schedule, we have 
\begin{align}
A_{\ell}(f_{\ell}-f^*) + C_{\ell}\|\bm{g}_{\ell}\|^2 + \frac{1}{2}\|\bm{x}_{\ell}-\bm{x}^*\|^2 
\leq 
\frac{1}{2}\|\bm{x}_1 - \bm{x}^*\|^2 + \sum_{i=1}^{\ell-1}\alpha_i \left( f_i-f^* -\langle\bm{g}_{i},\bm{x}_{i}-\bm{x}^{*}\rangle) + \frac{1}{2} \|\bm{g}_i\|^2 \right).\label{eq:sw0}
\end{align}
Recall from the basic properties \eqref{eq:basic-inequalities-f} that
\begin{align}
 & f_i - f_{\ell} \leq \langle\bm{g}_{i},\bm{x}_{i}-\bm{x}_{\ell}\rangle -\frac{1}{2}\|\bm{g}_i-\bm{g}_{\ell}\|^2  ,\nonumber
 \\ & f_i - f_0 \leq \langle\bm{g}_{i},\bm{x}_{i}-\bm{x}_{0}\rangle - \frac{1}{2}\|\bm{g}_i-\bm{g}_{0}\|^2 ,\nonumber
\end{align}
which allow us to derive
 \begin{align}
 & f_{i}-f^{*}-\langle\bm{g}_{i},\bm{x}_{i}-\bm{x}^{*}\rangle+\frac{1}{2}\|\bm{g}_{i}\|^{2}\nonumber\\
 & \qquad\leq f_{\ell}-f^{*}-\langle\bm{g}_{i},\bm{x}_{\ell}-\bm{x}_{i}+\bm{x}_{i}-\bm{x}^{*}\rangle+\frac{1}{2}\|\bm{g}_{i}\|^{2}-\frac{1}{2}\|\bm{g}_{i}-\bm{g}_{\ell}\|^{2}\nonumber\\
 & \qquad=f_{\ell}-f^{*}-\langle\bm{g}_{i},\bm{x}_{\ell}-\bm{x}^{*}\rangle+\langle\bm{g}_{i},\bm{g}_{\ell}\rangle-\frac{1}{2}\|\bm{g}_{\ell}\|^{2},
 %\qquad \text{and}
 \nonumber
 %\\
\end{align}
and similarly,
\begin{align}
 & f_{i}-f^{*}-\langle\bm{g}_{i},\bm{x}_{i}-\bm{x}^{*}\rangle+\frac{1}{2}\|\bm{g}_{i}\|^{2}\nonumber\\
 %& \qquad\leq f_{0}-f^{*}-\langle\bm{g}_{i},\bm{x}_{0}-\bm{x}_{i}+\bm{x}_{i}-\bm{x}^{*}\rangle+\frac{1}{2}\|\bm{g}_{i}\|^{2}-\frac{1}{2}\|\bm{g}_{i}-\bm{g}_{0}\|^{2}\nonumber\\
 & \qquad \leq f_{0}-f^{*}-\langle\bm{g}_{i},\bm{x}_{0}-\bm{x}^{*}\rangle+\langle\bm{g}_{i},\bm{g}_{0}\rangle-\frac{1}{2}\|\bm{g}_{0}\|^{2}.\nonumber
\end{align}
As a result,  we can take advantage of these properties to deduce that
\begin{subequations}
\begin{align}
 & \sum_{i=1}^{\ell-1}\alpha_{i}\left(f_{i}-f^{*}-\langle\bm{g}_{i},\bm{x}_{i}-\bm{x}^{*}\rangle+\frac{1}{2}\|\bm{g}_{i}\|^{2}\right)\nonumber\\
 & \quad\leq A_{\ell}(f_{\ell}-f^{*})-\sum_{i=1}^{\ell}\alpha_{i}\langle\bm{g}_{i},\bm{x}_{\ell}-\bm{x}^{*}\rangle+\sum_{i=1}^{\ell-1}\alpha_{i}\langle\bm{g}_{i},\bm{g}_{\ell}\rangle-\frac{A_{\ell}}{2}\|\bm{g}_{\ell}\|^{2}\nonumber\\
 & \quad=A_{\ell}(f_{\ell}-f^{*})-\langle\bm{x}_{1}-\bm{x}_{\ell},\bm{x}_{\ell}-\bm{x}^{*}\rangle+\langle\bm{x}_{1}-\bm{x}_{\ell},\bm{g}_{\ell}\rangle-\frac{A_{\ell}}{2}\|\bm{g}_{\ell}\|^{2},\label{eq:sw1}
\end{align}
and similarly,
\begin{align}
 & \sum_{i=1}^{\ell-1}\alpha_{i}\left(f_{i}-f^{*}-\langle\bm{g}_{i},\bm{x}_{i}-\bm{x}^{*}\rangle+\frac{1}{2}\|\bm{g}_{i}\|^{2}\right)\nonumber\\
 %& \qquad\leq A_{\ell}(f_{0}-f^{*})-\sum_{i=1}^{\ell}\alpha_{i}\langle\bm{g}_{i},\bm{x}_{0}-\bm{x}^{*}\rangle+\sum_{i=1}^{\ell-1}\alpha_{i}\langle\bm{g}_{i},\bm{g}_{0}\rangle-\frac{A_{\ell}}{2}\|\bm{g}_{0}\|^{2}\nonumber\\
 & \quad\leq A_{\ell}(f_{0}-f^{*})-\langle\bm{x}_{1}-\bm{x}_{\ell},\bm{x}_{0}-\bm{x}^{*}\rangle+\langle\bm{x}_{1}-\bm{x}_{\ell},\bm{g}_{0}\rangle-\frac{A_{\ell}}{2}\|\bm{g}_{0}\|^{2}.\label{eq:sw2}
\end{align}
\end{subequations}
Combine \eqref{eq:sw1} and \eqref{eq:sw2} to arrive at
\begin{align}
 & \sum_{i=1}^{\ell-1}\alpha_{i}\bigg(f_{i}-f^{*}-\langle\bm{g}_{i},\bm{x}_{i}-\bm{x}^{*}\rangle+\frac{1}{2}\|\bm{g}_{i}\|^{2}\bigg)\nonumber\\
 & \leq\frac{1}{2}\left(A_{\ell}(f_{0}+f_{\ell}-2f^{*})-\langle\bm{x}_{1}-\bm{x}_{\ell},\bm{x}_{0}+\bm{x}_{\ell}-2\bm{x}^{*}\rangle+\langle\bm{x}_{1}-\bm{x}_{\ell},\bm{g}_{0}+\bm{g}_{\ell}\rangle-\frac{A_{\ell}}{2}\|\bm{g}_{\ell}\|^{2}-\frac{A_{\ell}}{2}\|\bm{g}_{0}\|^{2}\right)\nonumber\\
 & =\frac{A_{\ell}}{2}\left(f_{0}+f_{\ell}-2f^{*}-\frac{\|\bm{g}_{\ell}\|^{2}}{2}-\frac{\|\bm{g}_{0}\|^{2}}{2}\right)\nonumber\\
 & \qquad\qquad-\frac{1}{2}\langle\bm{x}_{1}-\bm{x}_{\ell},\bm{x}_{1}+\alpha\bm{g}_{0}+\bm{x}_{\ell}-2\bm{x}^{*}\rangle+\frac{1}{2}\langle\bm{x}_{0}-\bm{x}_{\ell},\bm{g}_{0}+\bm{g}_{\ell}\rangle-\frac{1}{2}\alpha\big(\langle\bm{g}_{0},\bm{g}_{\ell}\rangle+\|\bm{g}_{0}\|^{2}\big)\nonumber\\
 & =\frac{A_{\ell}}{2}\left(f_{0}+f_{\ell}-2f^{*}-\frac{\|\bm{g}_{\ell}\|^{2}}{2}-\frac{\|\bm{g}_{0}\|^{2}}{2}\right)\nonumber\\
 & \qquad\qquad-\frac{1}{2}\langle\bm{x}_{1}-\bm{x}_{\ell},\bm{x}_{1}+\bm{x}_{\ell}-2\bm{x}^{*}\rangle-\frac{1}{2}\alpha\langle\bm{g}_{0},\bm{x}_{1}-\bm{x}_{\ell}\rangle+\frac{1}{2}\langle\bm{x}_{0}-\bm{x}_{\ell},\bm{g}_{0}+\bm{g}_{\ell}\rangle-\frac{1}{2}\alpha\big(\langle\bm{g}_{0},\bm{g}_{\ell}\rangle+\|\bm{g}_{0}\|^{2}\big)\nonumber\\
 & =\frac{A_{\ell}}{2}\left(f_{0}+f_{\ell}-2f^{*}-\frac{\|\bm{g}_{\ell}\|^{2}}{2}-\frac{\|\bm{g}_{0}\|^{2}}{2}\right)\nonumber\\
 & \qquad\qquad-\frac{1}{2}\big(\|\bm{x}_{1}-\bm{x}^{*}\|^{2}-\|\bm{x}_{\ell}-\bm{x}^{*}\|^{2}\big)-\frac{1}{2}\alpha\langle\bm{g}_{0},\bm{x}_{1}-\bm{x}_{\ell}\rangle+\frac{1}{2}\langle\bm{x}_{0}-\bm{x}_{\ell},\bm{g}_{0}+\bm{g}_{\ell}\rangle-\frac{1}{2}\alpha\big(\langle\bm{g}_{0},\bm{g}_{\ell}\rangle+\|\bm{g}_{0}\|^{2}\big).\nonumber
\end{align}
Adding this inequality and \eqref{eq:sw0}, we further reach
\begin{align}
 & A_{\ell}(f_{\ell}-f^{*})+C_{\ell}\|\bm{g}_{\ell}\|^{2}\nonumber\\
 & \leq\frac{1}{2}A_{\ell}\left(f_{0}+f_{\ell}-2f^{*}-\frac{\|\bm{g}_{\ell}\|^{2}}{2}-\frac{\|\bm{g}_{0}\|^{2}}{2}\right)-\frac{1}{2}\alpha\langle\bm{g}_{0},\bm{x}_{1}-\bm{x}_{\ell}\rangle+\frac{1}{2}\langle\bm{x}_{0}-\bm{x}_{\ell},\bm{g}_{0}+\bm{g}_{\ell}\rangle-\frac{1}{2}\alpha\big(\langle\bm{g}_{0},\bm{g}_{\ell}\rangle+\|\bm{g}_{0}\|^{2}\big).\nonumber
\end{align}
Rearrange terms to arrive at
\begin{align}
 & A_{\ell}(f_{0}-f_{\ell})\nonumber\\
 & \geq2C_{\ell}\|\bm{g}_{\ell}\|^{2}+\frac{1}{2}A_{\ell}(\|\bm{g}_{\ell}\|^{2}+\|\bm{g}_{0}\|^{2})+\alpha\langle\bm{g}_{0},\bm{x}_{1}-\bm{x}_{\ell}\rangle-\langle\bm{x}_{0}-\bm{x}_{\ell},\bm{g}_{0}+\bm{g}_{\ell}\rangle+\alpha\langle\bm{g}_{0},\bm{g}_{\ell}\rangle+\alpha\|\bm{g}_{0}\|^{2}\nonumber\\
 & =2C_{\ell}\|\bm{g}_{\ell}\|^{2}+\frac{1}{2}A_{\ell}(\|\bm{g}_{\ell}\|^{2}+\|\bm{g}_{0}\|^{2})+\alpha\langle\bm{g}_{0},\bm{x}_{0}-\alpha\bm{g}_{0}-\bm{x}_{\ell}\rangle-\langle\bm{x}_{0}-\bm{x}_{\ell},\bm{g}_{0}+\bm{g}_{\ell}\rangle+\alpha\langle\bm{g}_{0},\bm{g}_{\ell}\rangle+\alpha\|\bm{g}_{0}\|^{2}\nonumber\\
 & =2C_{\ell}\|\bm{g}_{\ell}\|^{2}+\frac{1}{2}A_{\ell}(\|\bm{g}_{\ell}\|^{2}+\|\bm{g}_{0}\|^{2})+\alpha\langle\bm{g}_{0},\bm{x}_{0}-\bm{x}_{\ell}\rangle-\langle\bm{x}_{0}-\bm{x}_{\ell},\bm{g}_{0}+\bm{g}_{\ell}\rangle+\alpha\langle\bm{g}_{0},\bm{g}_{\ell}\rangle+\alpha\|\bm{g}_{0}\|^{2}-\alpha^{2}\|\bm{g}_{0}\|^{2}\nonumber\\
 & =2C_{\ell}\|\bm{g}_{\ell}\|^{2}+\frac{1}{2}A_{\ell}(\|\bm{g}_{\ell}\|^{2}+\|\bm{g}_{0}\|^{2})+\langle\bm{x}_{0}-\bm{x}_{\ell},(\alpha-1)\bm{g}_{0}-\bm{g}_{\ell}\rangle+\alpha\langle\bm{g}_{0},\bm{g}_{\ell}\rangle+\alpha\|\bm{g}_{0}\|^{2}-\alpha^{2}\|\bm{g}_{0}\|^{2}.\label{eq:sw4}
\end{align}

The next step is to bound the term $\langle\bm{x}_{0}-\bm{x}_{\ell},(\alpha-1)\bm{g}_{0}-\bm{g}_{\ell}\rangle+ \alpha \langle \bm{g}_0, \bm{g}_{\ell}\rangle$. Towards this, we recall from \eqref{eq:basic-inequalities-f} that
\begin{align}
 (\alpha - 1)(f_0 - f_{\ell} )
 &\leq (\alpha-1) \langle \bm{g}_0, \bm{x}_{0}-\bm{x}_{\ell}\rangle - \frac{\alpha-1}{2} \|\bm{g}_0-\bm{g}_{\ell}\|^2  ;\nonumber
\\  f_{\ell}-f_0 &\leq -\langle \bm{g}_{\ell}, \bm{x}_{0}-\bm{x}_{\ell}\rangle - \frac{1}{2}\|\bm{g}_0-\bm{g}_{\ell}\|^2 .\nonumber
\end{align}
Adding the preceding two inequalities gives
\begin{align}
(\alpha -2)(f_0 -f_{\ell}) & \leq  \langle\bm{x}_{0}-\bm{x}_{\ell},(\alpha-1)\bm{g}_{0}-\bm{g}_{\ell}\rangle -\frac{\alpha}{2}\|\bm{g}_0-\bm{g}_{\ell}\|^2 \nonumber
\\ & = \langle\bm{x}_{0}-\bm{x}_{\ell},(\alpha-1)\bm{g}_{0}-\bm{g}_{\ell}\rangle -\frac{\alpha}{2}(\|\bm{g}_0\|^2+\|\bm{g}_{\ell}\|^2)+ \alpha \langle\bm{g}_0, \bm{g}_{\ell}\rangle ,\nonumber
\end{align}
thus indicating that
\begin{align}
\langle\bm{x}_{0}-\bm{x}_{\ell},(\alpha-1)\bm{g}_{0}-\bm{g}_{\ell}\rangle+\alpha\langle\bm{g}_{0},\bm{g}_{\ell}\rangle & \geq(\alpha-2)(f_{0}-f_{\ell})+\frac{\alpha}{2}(\|\bm{g}_{0}\|^{2}+\|\bm{g}_{\ell}\|^{2}).
\end{align}
Substitution into \eqref{eq:sw4} then leads to
\begin{align*}
A_{\ell}(f_{0}-f_{\ell}) \geq  2C_{\ell}\|\bm{g}_{\ell}\|^2 +  \frac{1}{2}A_{\ell}(\|\bm{g}_{\ell}\|^2 + \|\bm{g}_{0}\|^2) +  (\alpha -2)(f_0 -f_{\ell})  +  \frac{\alpha}{2}(\|\bm{g}_{0}\|^2+\|\bm{g}_{\ell}\|^2) +\alpha \|\bm{g}_0\|^2- \alpha^2 \|\bm{g}_0\|^2.
\end{align*}
Rearranging terms and using $C_{\ell}=\frac{A_{\ell}(A_{\ell}+1)}{2}$, we are left with
\begin{align}
(A_{\ell}+2 -\alpha )(f_0 -f_{\ell})\geq  \left(A_{\ell}^2 + \frac{3A_{\ell}}{2} + \frac{\alpha}{2}\right) \|\bm{g}_{\ell}\|^2 + \left( \frac{A_{\ell}}{2}+\frac{\alpha}{2}+\alpha  -\alpha^2 \right)\|\bm{g}_0\|^2.
\end{align}
Dividing both sides of the above display by $(A_{\ell}+2 -\alpha)$, 
we conclude the proof.

\end{document}